\documentclass{amsart}

\usepackage{times}
\usepackage{algorithmic}
\usepackage{algorithm}
\usepackage{graphicx}
\usepackage{subfigure}
\usepackage{bbm}
\usepackage{amsmath, amsthm, amssymb}

\usepackage{amssymb,amsmath}
\usepackage{algorithm}
\usepackage{algorithmic}
\usepackage{graphicx}
\usepackage{bm}
\usepackage{amsthm}
\usepackage{bbold}
\usepackage{multirow}
\usepackage{adjustbox}

\newtheorem{proposition}{Proposition}
\newtheorem{corollary}{Corollary}
\newtheorem{theorem}{Theorem}

\def\bfx{\mathbf{x}}
\def\bfy{\mathbf{y}}
\def\bfz{\mathbf{z}}

\def\bfc{\mathbf{c}}
\def\bfp{\mathbf{p}}

\def\bfC{\mathbf{C}}

\def\bfX{\mathbf{X}}
\def\bfY{\mathbf{Y}}
\def\bfD{\mathbf{D}}
\def\bfW{\mathbf{W}}

\def\bfP{\mathbf{P}}
\def\bfL{\mathbf{L}}
\def\bfI{\mathbf{I}}
\def\bfone{\mathbf{1}}
\def\bfZ{\mathbf{Z}}

\def\H{\mathcal{H}}
\def\P{\mathcal{P}}

\def\E{\mathcal{E}}

\def\G{\mathcal{G}}
\def\V{\mathcal{V}}
\def\N{\mathcal{N}}
\def\S{\mathcal{S}}

\def\T{\mathcal{T}}
\def\X{\mathcal{X}}
\def\Y{\mathcal{Y}}
\def\W{\mathcal{W}}

\usepackage{color}

\title{A Novel Regularized Principal Graph Learning Framework on Explicit Graph Representation}
\author[Q. Mao]{Qi Mao}
\author[L. Wang] {Li Wang}
\author[Ivor W. Tsang]{Ivor W. Tsang}
\author[Y. Sun]{Yijun Sun}
\thanks{Q. Mao is with Bioinformatics Lab, The State University of New York at Buffalo, Buffalo, NY 14201, USA.
		E-mail: maoq1984@gmail.com \\ 
		L. Wang is with Department of Mathematics, Statistics, and Computer Science, University of Illinois at Chicago, Chicago, USA. E-mail: liwang8@uic.edu \\
		I.W. Tsang is with the Centre for Quantum Computation \& Intelligent Systems (QCIS), the University of Technology, Sydney, Australia.\\
		Y. Sun is with Department of Microbiology and Immunology, The State University of New York at Buffalo, Buffalo, NY 14201, USA.}

\begin{document}
\maketitle

\begin{abstract}
Many scientific datasets are of high dimension, and the analysis usually requires visual manipulation by retaining the most important structures of data. Principal curve is a widely used approach for this purpose. However, many existing methods work only for data with structures that are not self-intersected, which is quite restrictive for real applications. A few methods can overcome the above problem, but they either require complicated human-made rules for a specific task with lack of convergence guarantee and adaption flexibility to different tasks, or cannot obtain explicit structures of data. To address these issues, we develop a new regularized principal graph learning framework that captures the local information of the underlying graph structure based on reversed graph embedding. As showcases, models that can learn a spanning tree or a weighted undirected $\ell_1$ graph are proposed, and a new learning algorithm is developed that learns a set of principal points and a graph structure from data, simultaneously. The new algorithm is simple with guaranteed convergence. We then extend the proposed framework to deal with large-scale data. Experimental results on various synthetic and six real world datasets show that the proposed method compares favorably with baselines and can uncover the underlying structure correctly.
\end{abstract}

\section{Introduction}

In many fields of science, one often encounters observations represented as high-dimen\-sional vectors sampled from unknown  distributions. It is sometimes difficult to directly analyze data in the original space, and is desirable to perform data dimensionality reduction or associate data with some structured objects for further exploratory analysis. One example is the study of human cancer, which is a dynamic disease that develops over an extended time period through the accumulation of a series of genetic alterations \cite{vogelstein2013cancer}. The delineation of this dynamic process would provide critical insights into molecular mechanisms underlying the disease process, and inform the development of diagnostics, prognostics and targeted therapeutics. The recently developed high-throughput genomics technology has made it possible to measure the expression levels of all genes in tissues from thousands of tumor samples simultaneously. However, the delineation of the cancer progression path embedded in a high-dimensional genomics space remains a challenging problem \cite{Sun2014}. 

Principal component analysis (PCA) \cite{Jolliffe1986} is one of the most commonly used methods to visualize data in a low-dimensional space, but its linear assumption limits its general applications. Several nonlinear approaches based on the kernel trick have been proposed \cite{Scholkopf1999}, but they remain sub-optimal for detecting complex structures. Alternatively, if data dimensionality is high, manifold learning based on the local information of data can be effective. Examples include locally linear embedding (LLE) \cite{Saul2003} and Laplacian eigenmaps \cite{Belkin2001}. However, these methods generally require to construct a carefully tuned neighborhood graph as their performance heavily depends on the quality of constructed graphs. 

Another approach is principal curve, which was initially proposed as a nonlinear generalization of the first principal component line \cite{Hastie1989}. Informally, a principal curve is an infinitely differentiable curve with a finite length that passes through the middle of data. Several principal curve approaches have been proposed, including those that minimize certain types of risk functions such as the quantization error \cite{Hastie1989,Kegl2000,Smola2001,Sandilya2002,Gorban2009} and the negative log-likelihood function \cite{Tibshirani1992, Bishop1998}. To overcome the over-fitting issue, regularization is generally required. K\'{e}gl et al. \cite{Kegl2000} bounded the total length of a principal curve, and proved that the principal curve with a bounded length always exists if the data distribution has a finite second moment. Similar results were obtained by bounding the turns of a principal curve \cite{Sandilya2002}. Recently, the elastic maps approach \cite{Gorban2009} was proposed to regularize the elastic energy of a membrane. An alternative definition of a principal curve based on a mixture model was considered in \cite{Hastie1989}, where the model parameters are learned through maximum likelihood estimation and the regularization is achieved using the smoothness of coordinate functions. Generative topographic mapping (GTM) \cite{Bishop1998} was proposed to maximize the posterior probability of the data which is generated by a low-dimensional discrete grid mapped into the original space and corrupted by additive Gaussian noise. 
GTM provides a principled alternative to the self-organizing map (SOM) \cite{Kohonen1997} for which it is impossible to define an optimality criterion \cite{Erwin1992}.

Methods for learning a principal curve have been widely studied, but they are generally limited to learn a structure that does not intersect itself \cite{Hastie1989}. Only a few methods can handle complex principal objects. K\'{e}gl and Krzyzak \cite{Kegl2002} extended their polygonal line method \cite{Kegl2000} for skeletonization of handwritten digits. The principal manifold approach \cite{Gorban2007} extends the elastic maps approach \cite{Gorban2009} to learn a graph structure generated by graph grammar. A major drawback of the two methods is that they require either a set of predefined rules (specifically designed for handwritten digits \cite{Kegl2002}) or grammars with many parameters difficult to be tuned, which makes their implementations complicated and their adaptations to new datasets difficult. More importantly, their convergences are not guaranteed. Recently, a subspace constrained mean shift (SCMS) method \cite{Ozertem2011} was proposed that can obtain principal points for any given second-order differentiable density function, but it is still not trivial to transform a set of principal points to an explicit structure. In many real applications, an explicit structure uncovered from the given data is helpful for downstream analysis. This critical point will be illustrated in Section \ref{sec:motivation} and Section \ref{sec:experiments} in detail with various real world datasets.

Another line of research relevant to this work is structure learning, which has made a great success on constructing or learning explicit structures from data. The graph structures that are commonly used in graph-based clustering and semi-supervised learning are $k$-nearest neighbor graph and $\epsilon$-neighborhood graph \cite{Belkin2001}. Dramatic influences of two graphs on clustering techniques have been studied in \cite{Maier2009}. Since the $\epsilon$-neighborhood graph could result in disconnected components or subgraphs in the dataset or even isolated singleton vertices, the b-matching method is applied to learn a better $b$-nearest neighbor graph via loopy belief propagation \cite{Jebara2009}. However, it is improper to use a fixed neighborhood size since the curvature of manifold and the density of data points may be different in different regions of the manifold \cite{elhamifar2011sparse}. In order to alleviate these issues, a method for simultaneous clustering and embedding of data lying in multiple manifolds \cite{elhamifar2011sparse} was proposed using $\ell_2$ norm over the errors that measure the linear representation of every data point by using its neighborhood information. Similarly, $\ell_1$ graph was learned for image analysis using $\ell_1$ norm over the errors for enhancing the robustness of the learned graph \cite{Cheng2010}. Instead of learning directed graphs by using the above two methods, an integrated model for an undirected graph by imposing a sparsity penalty on a symmetric similarity matrix and a positive semi-definite constraint on Laplacian matrix was proposed \cite{lake2010discovering}. Although these methods have demonstrated the effectiveness on various problems, it is challenging to apply them for moderate-size data, let alone large-scale data, due to the high computational complexity. Moreover, it might be suboptimal by heuristically transforming a directed graph to an undirected graph for clustering and dimensionality reduction \cite{Cheng2010,elhamifar2011sparse}.


\begin{figure*}[!ht]
	\centering
	\begin{tabular}{@{}c@{}c@{}c}
		\begin{adjustbox}{valign=b}
			\begin{tabular}{c}
				\includegraphics[width=0.25\textwidth, height=1.3in]{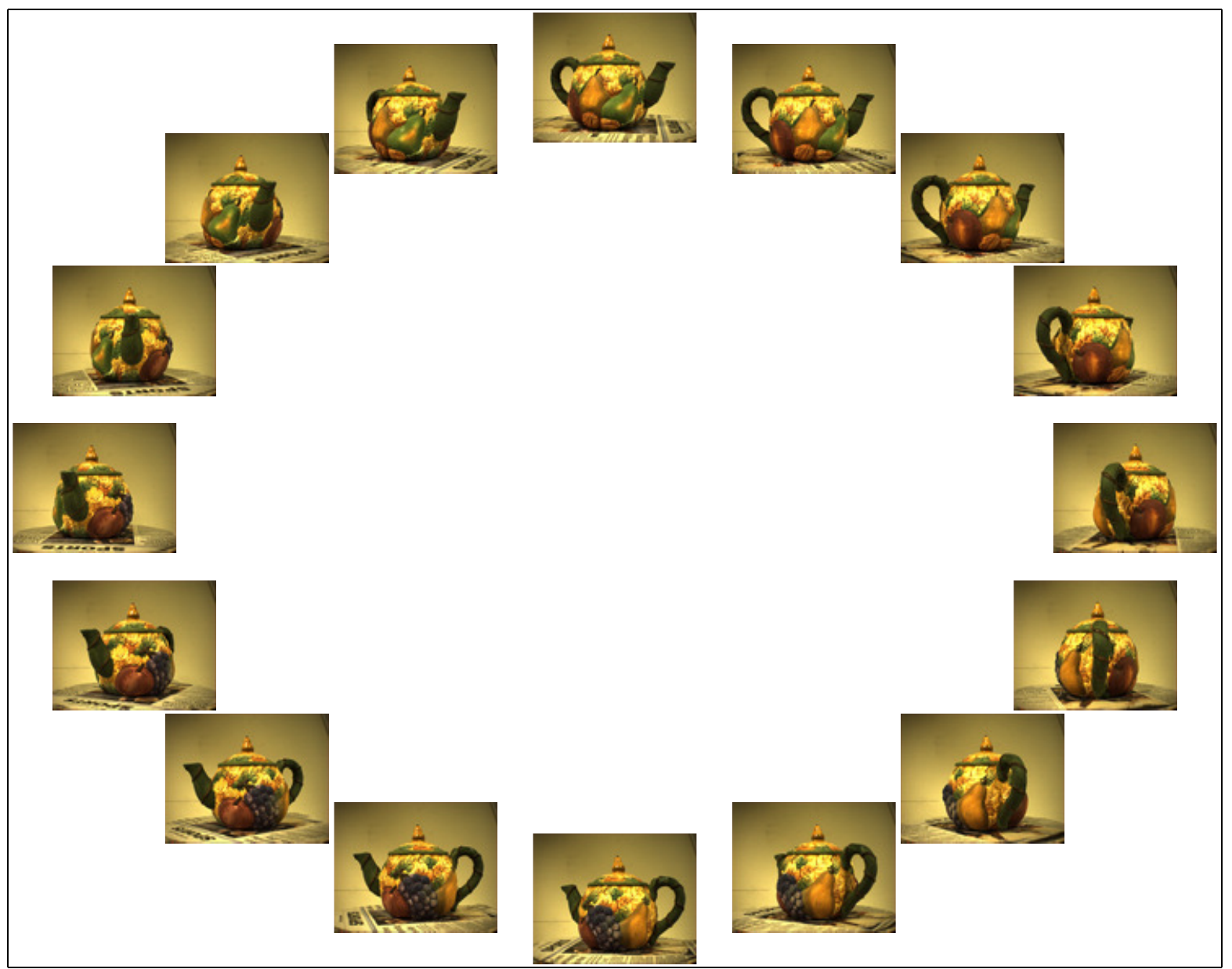} \\
				(a) $360^\circ$ rotation
			\end{tabular}
		\end{adjustbox}

		\begin{adjustbox}{valign=b}
			\begin{tabular}{c}
				\includegraphics[width=0.25\textwidth, height=1.3in]{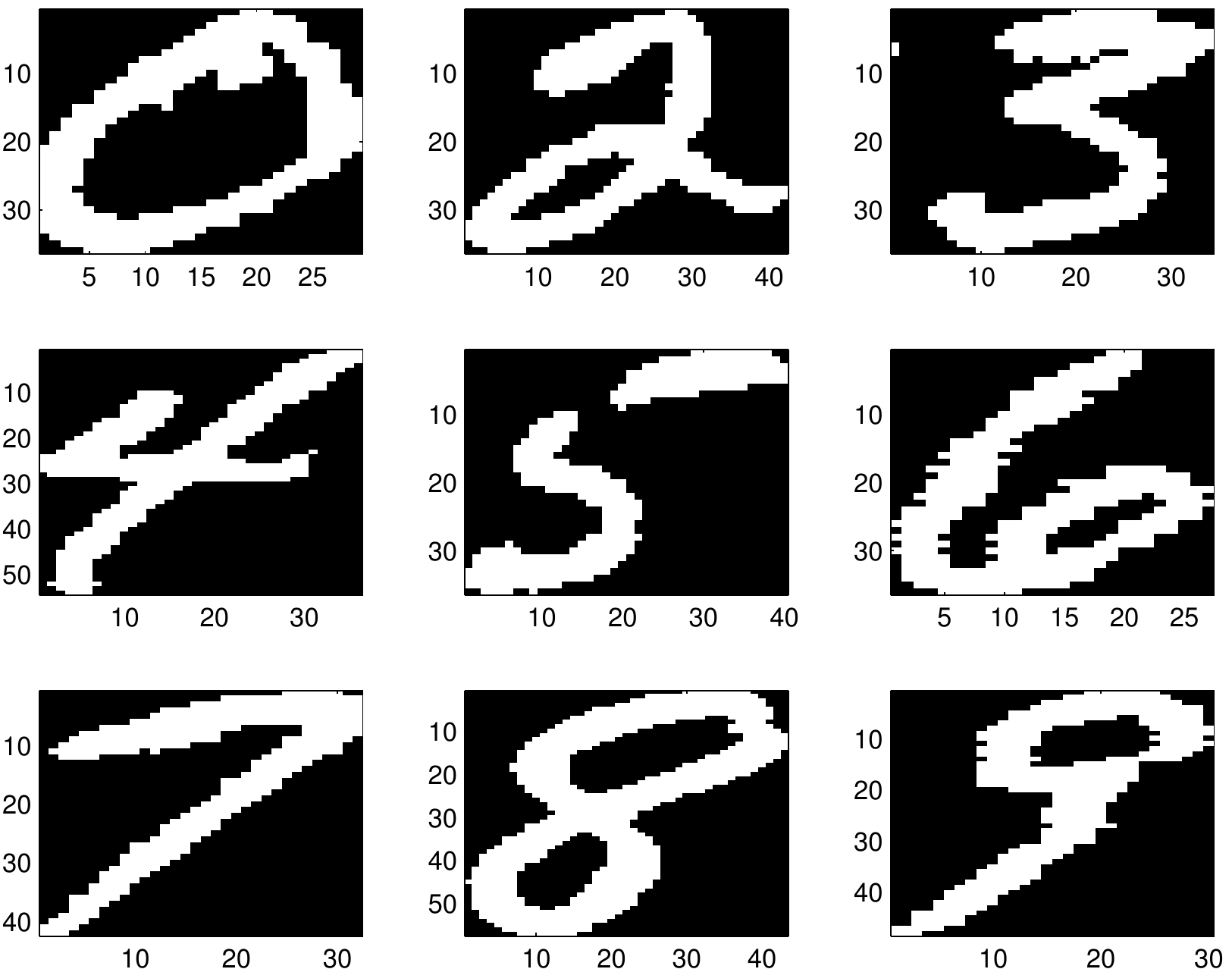} \\
				(b) complicated structures
			\end{tabular}
		\end{adjustbox}	
		
		\begin{adjustbox}{valign=b}
			\begin{tabular}{c}
				\includegraphics[width=0.32\textwidth]{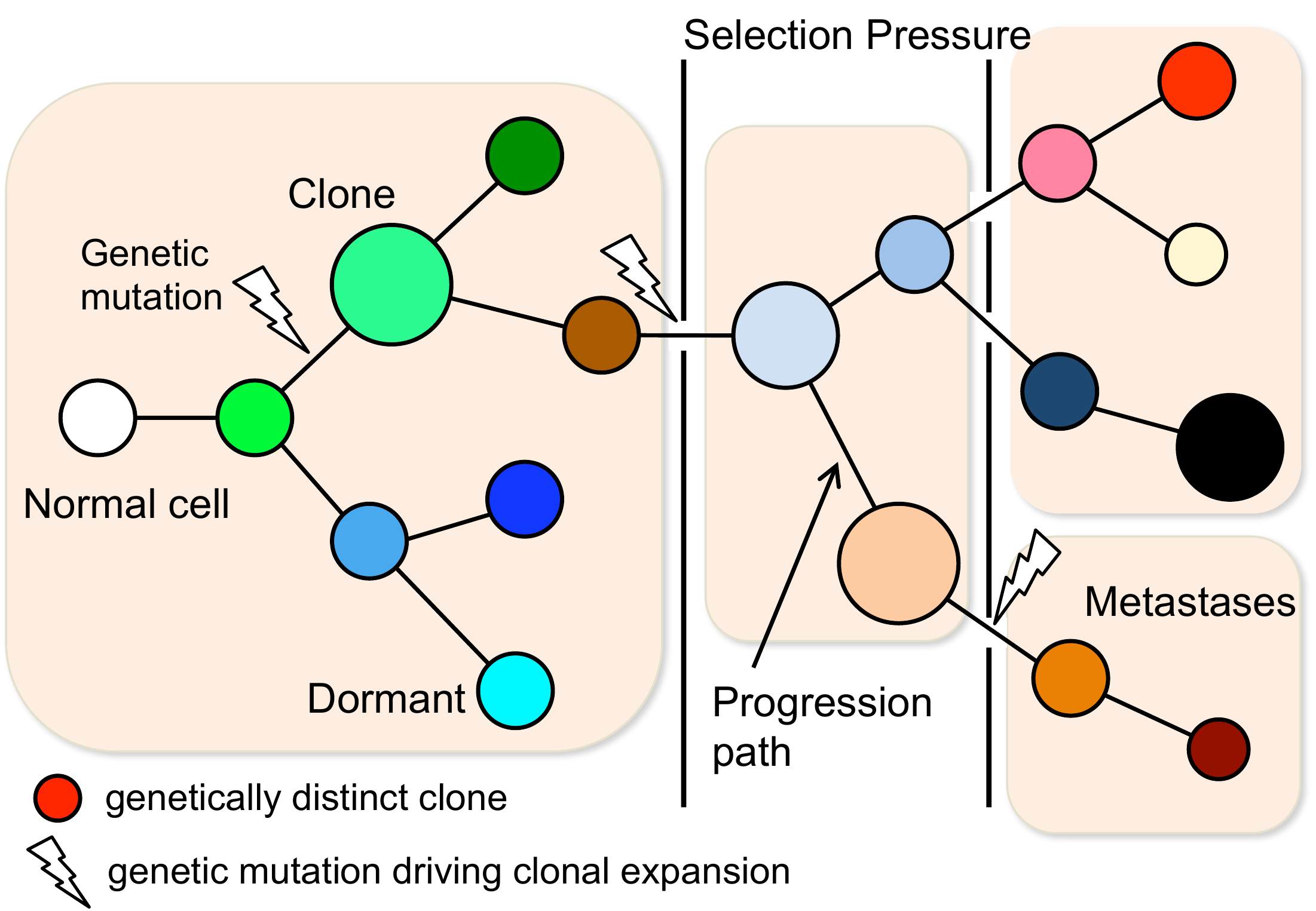} \\
				(c) cancer progression path
			\end{tabular}
		\end{adjustbox}
		
	\end{tabular}
	\caption{ Real world problems exhibiting a variety of graph structures. (a) $360^\circ$ rotation of teapot images forming a circle. (b) Optical character templates for different digits containing loops, bifurcations, and multiple disconnected components. (c) Branching architecture of cancer evolution. Selective pressures allow some tumor clones to expand while others become extinct or remain dormant. } \label{fig:evolution-tree}
\end{figure*}

This paper is an extension of our preliminary work \cite{Mao2015}, where we demonstrated the effectiveness of learning a minimum-cost spanning tree as a showcase for datasets of tree structures in certain applications. We move forward to take into account more complicated structures that might exist in most of real world datasets. Moreover, we propose two strategies to specifically overcome the issue of high-complexity of the proposed method for large-scale data. The main contributions of this paper are summarized as follows:
\begin{itemize}
	\item We propose a novel regularized principal graph learning framework that addresses the aforementioned limitations by learning a set of principal points and an explicit graph structure from data, simultaneously. To the best of our knowledge, this is the first work to represent principal objects using an explicit graph that is learned in a systematic way with a guaranteed convergence and is practical for large-scale data.
	
	\item Our novel formulation called reversed graph embedding for the representation of a principal graph facilitates the learning of graph structures, generalizes several existing methods, and possesses many advantageous properties. 
	
	\item In addition to spanning trees, a weighted undirected $\ell_1$ graph is proposed for modeling various types of graph structures, including curves, bifurcations, self-intersections, loops, and multiple disconnected components. This facilitates the proposed learning framework for datasets with complicated graphs.
	
	\item We further propose two strategies to deal with large-scale data. One is to use side-information as a priori to reduce the complexity of graph learning, the other is to learn a representative graph over a set of landmarks. Both strategies can be simultaneously incorporated into our framework for efficiently handling large-scale data.
	
	\item Extensive experiments are conducted for unveiling the underlying graph structures on a variety of datasets including various synthetic datasets and six real world applications consisting of various structures: hierarchy of facial expression images, progression path of breast cancer of microarray data, rotation circle of teapot images, smoothing skeleton of optical characters, and similarity patterns of digits on two large-scale handwritten digits databases.
	
\end{itemize}

The rest of the paper is organized as follows. We first illustrate the learning problem by using various real world datasets in Section \ref{sec:motivation}. In Section \ref{sec:pgr}, we present three key building blocks for the representation of principal graphs. Based on these key building blocks, we propose a new regularized principal graph learning framework in Section \ref{sec:rpg}. In Section \ref{sec:large}, we incorporate two strategies into the proposed framework to deal with large-scale datasets. Extensive experiments are conducted in Section \ref{sec:experiments}. Finally, we conclude this work in Section \ref{sec:conclusion}.

\section{Motivation of Structure Learning} \label{sec:motivation}

In many fields of science, experimental data resides in a high-dimensional space. However, the distance between two data points may not directly reflect the distance measured on the intrinsic structure of data. Hence, it is desirable to uncover the intrinsic structure of data before conducting further analysis.

It has been demonstrated that many high-dimensional datasets generated from real-world problems contain special structures embedded in their intrinsic dimensional spaces. One example is a collection of teapot images viewed from different angles \cite{Weinberger2006}. Each image contains $76 \times 101$ RGB pixels, so the pixel space has a dimensionality of $23,028$, but the intrinsic structure has only one degree of freedom: the angle of rotation. As shown in Fig. \ref{fig:evolution-tree}(a), distances computed by following an intrinsic circle are more meaningful than distances computed in the original space. Another example is  given in \cite{Song2007} where it is demonstrated that a collection of facial expression images ($217 \times 308$ RGB pixels) contains a two-layer hierarchical structure (Fig. 3(b) in \cite{Song2007}). The images from three facial expressions of one subject are grouped together to form the first layer, while all images from three subjects form the second layer. In other words, images of one subject should be distant from images of all other subjects, but images from three expressions of the same subject should be close. More complicated structures including loops, bifurcations, and multiple disconnected components as shown in Fig. \ref{fig:evolution-tree}(b) can be observed in optical character templates for identifying a smoothing skeleton of each character \cite{Kegl2002,Ozertem2011}. Other examples with specific intrinsic structures are also discussed in  \cite{Tenenbaum2000,Roweis2000}.

We are particularly interested in studying human cancer, a dynamic disease that develops over an extended time period. 
Once initiated from a normal cell, the advance to malignancy can to some extent be considered a Darwinian process - a multistep evolutionary process - that responds to selective pressure \cite{greaves2012clonal}. The disease progresses through a series of clonal expansions that result in tumor persistence and growth, and ultimately the ability to invade surrounding tissues and metastasize to distant organs. As shown in Fig. \ref{fig:evolution-tree}(c), the evolution trajectories inherent to cancer progression are complex and branching \cite{greaves2012clonal}. Due to the obvious necessity for timely treatment, it is not typically feasible to collect time series data to study human cancer progression [28]. However, as massive molecular profile data from excised tumor tissues (static samples) accumulates, it becomes possible to design integrative computation analyses that can approximate disease progression and provide insights into the molecular mechanisms of cancer.  We have previously shown that it is indeed possible to derive evolutionary trajectories from static molecular data, and  that breast cancer  progression can be represented by a high-dimensional manifold with multiple branches \cite{Sun2014}.

These concrete examples convince us that many datasets in a high-dimensional space can be represented by a certain intrinsic structure in a low dimensional space. Existing methods exploit the intrinsic structure of data implicitly either by learning a parametric mapping function or approximating a manifold via local geometric information of the observed data. However, these methods do not aim to directly learn an explicit form of an intrinsic structure in a low-dimensional space. In contrast, our proposed method is designed for this purpose to express the intrinsic structure of data by an explicit graph representation, which will be discussed in the following sections. 

\section{Principal Graph Representation} \label{sec:pgr}


Let $\X = \mathbb{R}^D$ be the input space and $\mathbb{x}=\{ \bfx_1, \ldots, \bfx_N \} \subset \X $ be a set of observed data points. We consider learning a projection function $h_{\G} \in \H$ such that latent points  $\mathbb{y}=\{ \bfy_1, \ldots, \bfy_K \}  \subset \Y = \mathbb{R}^d$ in $d$-dimensional space can trustfully represent $\mathbb{x}$, where $\H = \{ h_{\G}: \Y \rightarrow \X \}$ is a set of functions defined on a graph $\G$, and function $h_{\G}$ maps $\bfy_k$ to $h_{\G}(\bfy_k) \in \X, \forall k=1,\ldots,K$. Let $\G=(\V,\E, \bfW)$ be a weighted undirected graph, where $\V=\{ 1,\ldots,K\}$ is a set of vertices, $\E$ is a set of edges, and $\bfW \in  \mathbb{R}^{K \times K}$ is an edge weight matrix with the $(k, k')$th element denoted by $w_{k,k'}$ as the weight of edge $(k, k'), \forall k, k'$. Suppose that every vertex $k$ has one-to-one correspondence with latent point $\bfy_k \in \Y$, which resides on a manifold with an intrinsic dimension $d$. Next, we introduce three key components that form the building blocks of the proposed framework.

\subsection{Reversed Graph Embedding}

The most important component of the proposed framework is to model the relationship between graph $\G$ and latent points $\mathbb{y}$.
Given $\G$, weight $w_{k,k'}$ measures the similarity (or connection indicator) between two latent points $\bfy_k$ and $\bfy_{k'}$ in an intrinsic space $\Y$. Since the corresponding latent points $\{ \bfy_1, \ldots, \bfy_K \}$ are generally unknown in advance, we coin a new graph-based representation, namely Reversed Graph Embedding, in order to bridge the connection between graph structure $\G$ and corresponding data points in the input space. Specifically, our intuition is that if any two latent variables $\bfy_k$ and $\bfy_{k'}$ are neighbors on $\G$ with high similarity $w_{k,k'}$, two corresponding points $h_{\G}(\bfy_k)$ and $h_{\G}(\bfy_{k'})$ in the input space should be close to one another. 
To capture this intuition,  we propose to minimize the following objective function as an explicit representation of principal graphs, given by

\begin{small}\vspace{-0.1in}
	\begin{align}
	\Omega( \mathbb{y}, h_{\G}, \bfW ) = \sum_{ (k,k') \in \E } w_{k,k'} || h_{\G}(\bfy_k) - h_{\G}(\bfy_{k'}) ||_2^2, \label{eq:embed_rep}
	\end{align}
\end{small}\noindent
where $h_{\G} \in \H$ and $\mathbb{y} \in \Y$ are variables to be optimized for a given $\G$. At this moment, we cannot obtain interesting solutions by directly solving problem (\ref{eq:embed_rep}) with respect to $\{\mathbb{y}, h_{\G} \}$ due to multiple trivial solutions leading to zero objective value. However, objective function (\ref{eq:embed_rep}) possesses many interesting properties and can be used as a graph-based regularization for unveiling the intrinsic structure of the given data, which is detailed below. 

\subsubsection*{Relationship to Laplacian Eigenmap}
The objective function (\ref{eq:embed_rep}) can be interpreted from a reverse perspective of the well-known Laplacian eigenmap \cite{Belkin2001}. Denote $\bfY = [\bfy_1,\ldots,\bfy_N] \in \mathbb{R}^{d \times N}$. The Laplacian eigenmap solves the following optimization problem, 

\begin{small}\vspace{-0.1in}
	\begin{align*}
	\min_{\mathbb{y}}\sum_{i=1}^N \sum_{j=1}^N v_{i,j} ||\bfy_i - \bfy_j||_2^2: \bfY \bfD \bfY^T = \bfI, 
	\end{align*} 
\end{small}\noindent
where the similarity $v_{i,j}$ between $\bfx_i$ and $\bfx_j$ is computed in the input space $\X$ in order to capture the locality information of the manifold modeled by neighbors of the input data points, e.g., $v_{i,j} = \exp(- \frac{ || \bfx_i - \bfx_j ||_2^2 }{2 \sigma^2})$ using the heat kernel with bandwidth parameter $\sigma^2$ if $i$ is the neighbor of $j$ and $j$ is also the neighbor of $i$, and $v_{i,j}=0$ otherwise, and $\bfD = diag ([ \sum_{i=1}^N v_{i,j} ] ) \in \mathbb{R}^{N \times N}$ is a diagonal matrix.
The Euclidean distance between $\bfy_i$ and $\bfy_j$, i.e., $||\bfy_i - \bfy_j||_2$, are computed in the latent space $\Y$. 
Consequently, two multiplication terms of the objective functions of reverse graph embedding and Laplacian eigenmap are calculated in a different space: (i) weights $w_{i,j}$ and $v_{i,j}$ are computed in $\Y$ and $\X$, respectively; (ii) distances $|| h_{\G}(\bfy_k) - h_{\G}(\bfy_{k'}) ||_2$ and $||\bfy_i - \bfy_j||_2^2$ are computed in $\X$ and $\Y$, respectively. 

Our formulation can directly model the intrinsic structure of data, while Laplacian eigenmap approximates the intrinsic structure by using neighborhood information of observed data points. In other words, our proposed reversed graph embedding representation facilitates the learning of a graph structure from data. The weight $w_{k,k'}$ encodes the similarity or connectivity between vertices $k$ and $k'$ on $\G$. For example, if $w_{k,k'}=0$, it means that edge $(k,k')$ is absent from $\E$. In most cases, a dataset is given, but graph $\G$ is unknown. In these cases, it is necessary to automatically learn $\G$ from data. The objective function (\ref{eq:embed_rep}) is linear with respect to weight matrix $\bfW$. This linearity property facilitates the learning of the graph structure. We will discuss different representations of graphs based on linearity property in Section \ref{sec:sparse-graph}.

Another important difference is that the number of the latent variables $\mathbb{y}$ is not necessary equal to the number of input data points $\mathbb{x}$, i.e., $K \leq N$. This is useful to obtain a representative graph over a set of landmark points by compressing a large amount of input data points, and the representative graph can still trustfully represent the whole data. We will explore this property for large-scale principal graph learning in Section \ref{sec:landmark}. 


\subsubsection*{Harmonic Function of a General Graph}
We have discussed the properties of the reversed graph embedding by treating variables $\bfy_k$ and $h_{\G}$ as a single integrated variable $h_{\G}(\mathbf{y}_k), \forall k$. Actually, the optimal function $h_{\G} \in \H$ obtained by solving (\ref{eq:embed_rep}) is related to harmonic or pluriharmonic functions.
This can be further illustrated by the following observations. Let $\N_k$ be the neighbors of point $\bfy_k, \forall k$. For any given $\bfy_k$, problem (\ref{eq:embed_rep}) can be rewritten as

\begin{small}\vspace{-0.1in}
	\begin{align*}
	\min_{h_{\G}(\bfy_k)} \sum_{k' \in \N_k} w_{k,k'} || h_{\G}(\bfy_k) - h_{\G}(\bfy_{k'}) ||^2,
	\end{align*}
\end{small}\noindent
which has an analytic solution by fixing the rest of variables $\{h_{\G}(\bfz_j)\}_{j \not = m}$ given by

\begin{small} \vspace{-0.1in}
	\begin{align}
	h_{\G} (\bfy_k) = \frac{1}{\sum_{k' \in \N_k} w_{k,k'}}\sum_{k' \in \N_k} w_{k,k'} h_{\G}(\bfy_{k'}). \label{eq:harmonic_function}
	\end{align}
\end{small}\noindent
If equalities (\ref{eq:harmonic_function}) hold for all $k$, function $h_{\G}$ is a harmoinc function on $\G$ since its value in each nonterminal vertex is the mean of the values in the closest neighbors of this vertex \cite{Gorban2009}. The plurihamonic graphs defined in \cite{Gorban2009} impose penalty only on a subset of $k$-stars as, 

\begin{small}\vspace{-0.1in}
	\begin{align}
	\left|\left|h_{\G} (\bfz_m) - \frac{1}{\sum_{j \in \N_m} w_{m,j}}\sum_{j \in \N_m} w_{m,j} h_{\G}(\bfz_j)\right|\right|^2, \label{eq:LLE}
	\end{align}
\end{small}\noindent
where $|\N_{m}| = k, \forall m$. In contrast, our formulation (\ref{eq:embed_rep}) allows to flexibly incorporate any neighborhood structure existing in $\G$.  The connection of $h_{\G}$ to harmonic or pluriharmonic functions enriches the target function, which has been previously discussed in \cite{Gorban2009}.

\subsubsection*{Extension of Principal Curves to General Graphs}
Equation (\ref{eq:embed_rep}) generalizes various existing methods.
It is worth noting that the quantity $\Omega( \mathbb{y}, h_{\G}, \bfW )$ can be considered as the length of a principal graph in terms of the square of $\ell_2$ norm.  
In the case where $\G$ is a linear chain structure, $\Omega( \mathbb{y}, h_{\G}, \bfW )$ is same as the length of a polygonal line defined in \cite{Kegl2000}. However, general graphs or trees are more flexible than principal curves since the graph structure allows self-intersection. For principal graph learning, elastic map \cite{Gorban2009} also defines a penalty based on a given graph. However, based on the above discussion, it is difficult to solve problem (\ref{eq:LLE}) with respect to both the function $h_{\G}$ and the graph weights $\bfW$ within the elastic-maps framework. In contrast, the proposed reversed graph embedding leads to a simple and efficient algorithm to learn a principal tree or a weighted undirected $\ell_1$ graph with guaranteed convergence. This will be clarified in Section \ref{sec:rpg}.

\subsection{Data Grouping} \label{sec:data-fitting}

The second important component of the proposed framework is to measure the fitness of latent variables $\mathbb{y}$ to a given data $\mathbb{x}$ in terms of a given graph $\G$ and projection function $h_{\G}$. As the number of latent variables is not necessarily equal to the number of input data points, we assume that projected point $h_{\G}(\bfy_k)$ is a centroid of the $k$th cluster of $\mathbb{x}$ so that input data points with high similarity form a cluster. The empirical quantization error \cite{Smola2001} is widely used as the fitting criterion to be minimized for the optimal cluster centroids, and it is also frequently employed in principal curve learning methods \cite{Hastie1989,Kegl2000,Smola2001,Sandilya2002,Gorban2009}, given by

\begin{small}\vspace{-0.13in}
	\begin{align}
	\ell( {\mathbb{x}, \mathbb{y}}, h_{\G}) = \sum_{i=1}^N \min_{k=1,\ldots,K} || \bfx_i -  h_{\G}(\bfy_k) ||_2^2. \label{eq:quant-loss}
	\end{align}
\end{small}\noindent
Based on equation (\ref{eq:quant-loss}), we have the following loss functions.

{\bf Data Reconstruction Error:} If $K=N$, we can reformulate the equation (\ref{eq:quant-loss}) by reordering $\mathbb{y}$ as

\begin{small}\vspace{-0.13in}
	\begin{align}
	\ell_N( {\mathbb{x}, \mathbb{y}}, h_{\G}) = \sum_{i=1}^N || \bfx_i - h_{\G}(\bfy_i) ||_2^2. \label{eq:loss-N}
	\end{align}
\end{small}\noindent
This formulation can be interpreted as the negative log-likelihood of data $\mathbb{x}$ that are i.i.d drawn from a multivariate normal distribution with mean $h_{\G}(\bfy_i)$ and covariance matrix $\frac{1}{2} \bfI$.

{\bf K-means:} If $K < N$, we introduce an indicator matrix $\Pi \in \{0, 1\}^{N \times K}$ with the $(i,k)$th element $\pi_{i,k} = 1$ if $\bfx_i$ is assigned to the $k$th cluster with centroid $h_{\G}(\bfy_k)$, and $\pi_{i,k} = 0$ otherwise. Consequently, we have the following equivalent optimization problem 

\begin{small}\vspace{-0.13in}
	\begin{align}
	\ell_{\Pi}( {\mathbb{x}, \mathbb{y}}, h_{\G}) = \sum_{i=1}^N \sum_{k=1}^K \pi_{i,k} || \bfx_i -  h_{\G}(\bfy_k) ||_2^2, \label{eq:loss-kmeans}
	\end{align}
\end{small}\noindent
where $\sum_{k=1}^K \pi_{i,k} = 1$ and $\pi_{i,k} \in \{0,1\}, \forall i=1,\ldots,N$. This is the same as the optimization problem of $K$-means method that minimizes the objective function (\ref{eq:quant-loss}).

{\bf Generalized Empirical Quantization Error:} If $K \leq N$, a right stochastic matrix $\bfP \in [0,1]^{N \times K}$ with each row summing to one is introduced, that is, $\sum_{k=1}^K p_{i,k}=1, \forall i=1,\ldots, N$. This variant is equivalent to the above representation of indicator matrix $\Pi$ if an integer solution $\bfP$ is obtained. When $K$ is relatively large, $K$-means that minimizes (\ref{eq:loss-kmeans}) might generate many empty clusters. To avoid this issue, we introduce the soft assignment strategy by adding negative entropy regularization as

\begin{small}\vspace{-0.13in}
	\begin{align}
	\ell_{\bfP}( {\mathbb{x}, \mathbb{y}}, h_{\G}) = \sum_{i=1}^N \sum_{k=1}^K p_{i,k} \Big [ || \bfx_i -  h_{\G}(\bfy_k) ||_2^2 + \sigma \log p_{i,k} \Big], \label{eq:soft-kmeans}
	\end{align}
\end{small}\noindent
where $\sigma > 0$ is a regularization parameter.

In this paper, we use $\ell_{\bfP}( {\mathbb{x}, \mathbb{y}}, h_{\G})$ 
since it is a generalized version of the other two cases and also has a close relationship with the mean shift clustering method \cite{Cheng1995}, Gaussian mixture model, and the K-means method.
Below, we discuss these interesting properties in detail. The proofs of Proposition \ref{prop:mean-shift} and Proposition \ref{prop:reformulation} are given in Appendix.

\begin{proposition} \label{prop:mean-shift}
	The mean shift clustering \cite{Cheng1995} is equivalent to minimizing objective function (\ref{eq:soft-kmeans}) with respect to a left stochastic matrix $\bfP \in [0,1]^{N \times K}$ with each column summing to one.
\end{proposition}

\begin{proof}


Mean shift clustering \cite{Cheng1995} is a non-parametric model for locating modes of a density function. Let $\mathbb{c}=\{ \bfc_1,\ldots,\bfc_K \}$ be modes of $f(\mathbb{c})$ that we aim to find. The kernel density estimator is used to estimate the density function from $\mathbb{x}$, for any $ \bfc_k \in \mathbb{c}$, given by
	\begin{align*}
	f(\bfc_k | \mathbb{x}) = \frac{1}{N} \sum_{i=1}^N (\pi \sigma)^{-D/2} \exp \left( - \frac{||\bfx_i - \bfc_k||^2}{\sigma} \right),
	\end{align*}
where Gaussian kernel with covariance matrix $\frac{1}{2}\sigma \bfI$ is used and $\sigma > 0$. Since data points $\mathbb{x}$ are independent and identically distributed, we have a joint density function over $\mathbb{c}$ given by
	\begin{align*}
	f(\mathbb{c} | \mathbb{x}) = \prod_{k=1}^K f(\bfc_k | \mathbb{x}).
	\end{align*}
An optimization problem is naturally formed to find the modes by maximizing the joint density function with respect to $\mathbb{c}$ as
	\begin{align}
	\max_{\mathbb{c}} \sum_{k=1}^K \log \left[\sum_{i=1}^N \exp \left( - \frac{||\bfx_i - \bfc_k||^2}{\sigma}  \right) \right]. \label{op:mean-shift}
	\end{align}
	
On the other hand, we introduce the left stochastic matrix $\bfP \in \P_c$, where $$\P_c = \left\{ \sum_{i=1}^N p_{i,k} = 1, p_{i,k} \geq 0, \forall i, k \right\}.$$ Let $\bfc_k = h_{\G}(y_k), \forall k$ and $\bfp_k = [p_{1,k},\ldots,p_{N,k}]^T$. For each $k=1,\ldots,K$, minimizing equation of generalized empirical quantization error can be reformulated as the following optimization problem
	\begin{align*}
	\min_{\bfp_k} &~ \sum_{i=1}^N  p_{i,k}  \Big[ ||\bfx_i - \bfc_k||^2 + \sigma \log p_{i,k} \Big] \\
	\textrm{s.t.} &~ \sum_{i=1}^N p_{i,k} = 1, p_{i,k} \geq 0, \forall i.
	\end{align*}
By Lagrange duality theorem \cite{Boyd2004}, the KKT conditions for optimal solution $\bfp_k$ are given by
	\begin{align*}
	||\bfx_i - \bfc_k||^2 + \sigma (1 + \log p_{i,k}) + \alpha = 0, \forall k,\\
	\sum_{i=1}^N p_{i,k} = 1,p_{i,k} \geq 0 \forall i, k.
	\end{align*}
By combining above two equalities, we have the optimal solution as
	\begin{align}
	p_{i,k} = \frac{\exp \left( - {||\bfx_i - \bfc_k||^2}/{\sigma} \right)  }{ \sum_{i=1}^N \exp \left( - {||\bfx_i - \bfc_k||^2}/{ \sigma}  \right)  }, \forall i, k \label{eq:p}
	\end{align}
and the objective function to be minimized is given by
	\begin{align*}
	-\sigma \log \left[\sum_{i=1}^N \exp \left( - \frac{||\bfx_i - \bfc_k||^2}{\sigma}  \right) \right], \forall k.
	\end{align*}
By substituting (\ref{eq:p}) into (\ref{eq:soft-kmeans}), it becomes Problem (\ref{op:mean-shift}). The proof is completed.
\end{proof}

\begin{proposition} \label{prop:reformulation}
	Minimizing objective function (\ref{eq:soft-kmeans}) with respect to a right stochastic matrix $\bfP \in [0,1]^{N \times K}$ with each row summing to one can be interpreted as the Gaussian mixture model with uniform weights.
\end{proposition}

\begin{proof}
Let $\bfP \in \P_r$ be a right stochastic matrix where $\P_r = \left\{ \sum_{k=1}^K p_{i,k} = 1, p_{i,k} \geq 0, \forall i, k \right\}$, and $\mathbb{c}=\{ \bfc_1, \ldots, \bfc_N\}$ where $\bfc_i = h_{\G}(\bfy_i)$, and $\bfC = [\bfc_1,\ldots,\bfc_N] \in \mathbb{R}^{D \times N}$.
Following the proof of Proposition 1, we have the optimal solution of minimizing (\ref{eq:soft-kmeans}) with respect to $\bfP$ as
	\begin{align}
	p_{i,k} = \frac{\exp \left( - {||\bfx_i - \bfc_k||^2}/{\sigma} \right)  }{ \sum_{k=1}^K \exp \left( - {||\bfx_i - \bfc_k||^2}/{ \sigma}  \right)  }, \forall i, k, \label{eq:p_r}
	\end{align}
where the optimal solution $\mathbf{c}$ is achieved by solving the following optimization problem
	\begin{align}
	\max_{\mathbb{c}} \sigma \sum_{i=1}^N \log \left[\sum_{k=1}^K \exp \left( - \frac{||\bfx_i - \bfc_k||^2}{\sigma}  \right) \right]. \label{op:mgm}
	\end{align}
For a fixed $\sigma$, the above objective function is equivalent to the maximum likelihood function of $\mathbb{x}$ which are i.i.d. drawn from the mixture of Gaussian distributions consisting of $K$ components each which is a Gaussian distribution $\N(\bfx | \bfc_k, \frac{1}{2}\sigma \bfI)$ and discrete uniform distribution for each component, i.e. $\frac{1}{K}$.
\end{proof}

\begin{corollary}
	We have the following relationships among three loss functions: 
	\begin{enumerate}
		\item If $\sigma \rightarrow 0$, minimizing (\ref{eq:soft-kmeans}) is equivalent to minimizing (\ref{eq:loss-kmeans}) so that $\bfP=\Pi$.
		\item If $\sigma \rightarrow 0$ and $N = K$, minimizing (\ref{eq:soft-kmeans}) is equivalent to minimizing (\ref{eq:loss-N}).
	\end{enumerate}
\end{corollary}

According to the aforementioned results, we can briefly summarize the merits of function (\ref{eq:soft-kmeans}). First, empty clusters will never be created according to Proposition \ref{prop:reformulation} for any $K \leq N$. Second, the loss function takes the density of input data $\mathbb{x}$ into account. In other words, the centroids obtained by minimizing the loss function should reside in the high density region of the data. Third, the loss function makes the learning of graph structures computationally feasible for large-scale data in the case of $K < N$, which will be discussed in Section \ref{sec:large}.

\subsection{Latent Sparse Graph Learning} \label{sec:sparse-graph}
The third important component of the proposed framework is to learn a latent graph from data.
To achieve this goal, we investigate two special graph structures that can be learned from data by formulating the learning problems as linear programming. One is a tree structure, represented as a minimum-cost spanning tree, the other is a weighted undirected graph that is assumed to be sparse. Let $\mathbb{z} = \{ \bfz_1, \ldots, \bfz_K \} \subset \mathbb{R}^D$ be a dataset. Our task is to construct a weighted undirected graph $\G=(\V,\E, \bfW)$ with a cost $\phi_{k,k'}$ associated with edge $(k,k') \in \E, \forall k, k'$ by optimizing similarity matrix $\bfW \in \mathbb{R}^{K \times K}$ based on the assumption of the specific graph structure.

\subsubsection*{Minimum Cost Spanning Tree}

Let $\T = (\V, \E_{\T})$ be a tree with the minimum total cost and $\E_{\T}$ be the edges forming a tree.
In order to represent and learn a tree  structure, we consider $\bfW$ as binary indicator matrix where $w_{k,k'}=1$ if $(k,k') \in \E_{\T}$, and $w_{k,k'}=0$ otherwise.
The integer linear programming formulation of a minimum spanning tree (MST) can be written as, 
$\min_{\bfW \in \W_0} \sum_{k,k'} w_{k,k'} \phi_{k,k'}$,
where $\W_0 = \{ \bfW \in \{0,1\}^{N \times N}\} \cap \W'$ and $\W'= \{\bfW = \bfW^T\} \cap \{ \frac{1}{2} \sum_{k,k'} w_{k,k'} = |\V| -1, w_{k,k}=0, \forall k \} \cap \{ \frac{1}{2} \sum_{ k \in \S, k' \in \S} w_{k,k'} \leq |\S|-1, \forall \S \subseteq \V  \}$. The first constraint of $\W'$ enforces the symmetric connection of undirected graph, e.g. $w_{k,k'} = w_{k',k}$.
The second constraint states that the spanning tree only contains $|\V|-1$ edges. The third constraint imposes the acyclicity and connectivity properties of a tree. It is difficult to solve an integer programming problem directly. Thus, we resort to a relaxed problem by letting $w_{k,k'} \geq 0$, that is,

\begin{small}\vspace{-0.1in}
	\begin{align}
	\min_{\bfW \in \W_{\T}} &~ \sum_{k=1}^K \sum_{k'=1}^K w_{k,k'} \phi_{k,k'}, \label{op:spanning-tree}
	\end{align}
\end{small}\noindent
where the set of linear constraints is given by $\W_{\T} = \left\{ \bfW \geq 0 \right\} \cap \W'$.
Problem (\ref{op:spanning-tree}) can be readily solved by the Kruskal's algorithm \cite{Kruskal1958,Cheung2008}. 


\subsubsection*{Weighted Undirected $\ell_1$ Graph}
A complete graph with imposed sparse constraints over edge weights $\bfW$ called $\ell_1$ graph is considered for learning a sparse graph.
An $\ell_1$ graph is motivated from the intuition that one data point can usually be represented by a linear combination of other data points if they are similar to this data point. In other words, two data points $\bfz_k$ and $\bfz_{k'}$ are similar if edge weight $w_{k,k'}$ is a positive value, which can be treated as a similarity measurement between two data points. 
Consequently, the coefficient of linear combination is coincident with the edge weight on a graph.
The system of linear equations $\bfz_k = \sum_{k'=1}^K w_{k',k} \bfz_{k'}, \forall k$, where $\bfz_k$ is the vector to be approximated, $\bfW$ is the edge weight matrix, and $\mathbb{z}$ is the overcomplete dictionary with $K$ bases.
In practice, there may exist noises on certain elements of $\bfz_k$, and a natural way is to estimate the edge weights by tolerating these errors. We introduce error $\bm{\xi}_k \in \mathbb{R}^{D}$ to each linear equation and the equality constraints are formulated as
%
	$\bfz_k = \sum_{k'=1, k' \not= k }^K w_{k',k} \bfz_{k'} + \bm{\xi}_k, \forall k$.
In order to learn a similarity measurement for an undirected graph, we impose nonnegative and symmetric constraints on $\bfW$, i.e., $\{ \bfW \geq 0, \bfW = \bfW^T \}$. Generally, a sparse solution is more robust and facilitates the consequent identification of the test sample $\bfz_k$. Following the recent results on the sparse representation problem, the convex $\ell_1$-norm minimization can effectively recover the sparse solution \cite{elhamifar2011sparse,Cheng2010}. Let $\bfZ = [\bfz_1,\ldots,\bfz_K] \in \mathbb{R}^{D \times K}$. We propose to learn $\bfW$ by solving the following linear programming problem

\begin{small}\vspace{-0.1in}
	\begin{align}
	\min_{\bfW, \{\xi_k\}_{k=1}^K } &~ \sum_{k=1}^K \sum_{k'=1}^K w_{k,k'} \phi_{k,k'} + \lambda \sum_{k=1}^K ||\bm{\xi}_k||_1 \label{op:l1-sparse-graph}\\
	\textrm{s.t.} &~ \bfz_k = \sum_{k'=1, k' \not= k }^K w_{k',k} \bfz_{k'} + \bm{\xi}_k, \forall k \nonumber \\
	&~ \bfW \geq 0, \bfW = \bfW^T, diag(\bfW)=0. \nonumber
	\end{align}
\end{small}\noindent
where $\lambda > 0$ is a parameter for weighting all errors and $||\bm{\xi}_k||_1 = \sum_{m=1}^D |\xi_k^m|$ with $\bm{\xi}_k = [\xi_k^1, \ldots, \xi_k^D]^T$. The problem (\ref{op:l1-sparse-graph}) is different from methods \cite{elhamifar2011sparse,Cheng2010} that learn directed $\ell_1$ graphs, and different from the method \cite{lake2010discovering} that learns undirected graphs using the probabilistic model with Gaussian Markov random fields. Moreover, data points $\mathbb{z}$ can be different from the data points that are used to compute costs.

\subsubsection*{Generalized Graph Structure Learning}
For the ease of representation, we use a unified formulation $g( \bfW, \mathbb{z}, \bm{\Phi} )$  for learning a similarity graph, where $\mathbb{z}$ is a given dataset, $\bfW$ is the similarity matrix of a graph over the given dataset, and $\bm{\Phi}$ is a cost matrix with the $(k,k')$th entry as $\phi_{k,k'}$. The feasible space of $\bfW$ is denoted as $\W$. Specifically, we solve the following generalized graph structure learning problem such that any graph learning problem that can be formulated as linear programming with constraints represented as $\W$ can be used in the following proposed framework, given by

\begin{small}\vspace{-0.1in}
	\begin{align}
	\min_{\bfW \in \W} g( \bfW, \mathbb{z}, \bm{\Phi} ). \label{op:general-graph}
	\end{align}
\end{small}\noindent
It is easy to identify the correspondences of problem (\ref{op:general-graph}) to problems (\ref{op:spanning-tree}) and (\ref{op:l1-sparse-graph}).

\section{Regularized Principal Graph} \label{sec:rpg}

By combining the three building blocks discussed in Section \ref{sec:pgr}, we are ready to propose a unified framework for learning a principal graph. We use the alternate convex search algorithm to solve the proposed formulations by simultaneously learning a set of principal points and an undirected graph with guaranteed convergence.

\subsection{Proposed Formulation}

\begin{figure}[t]
	\centering
	\includegraphics[width=0.95\textwidth]{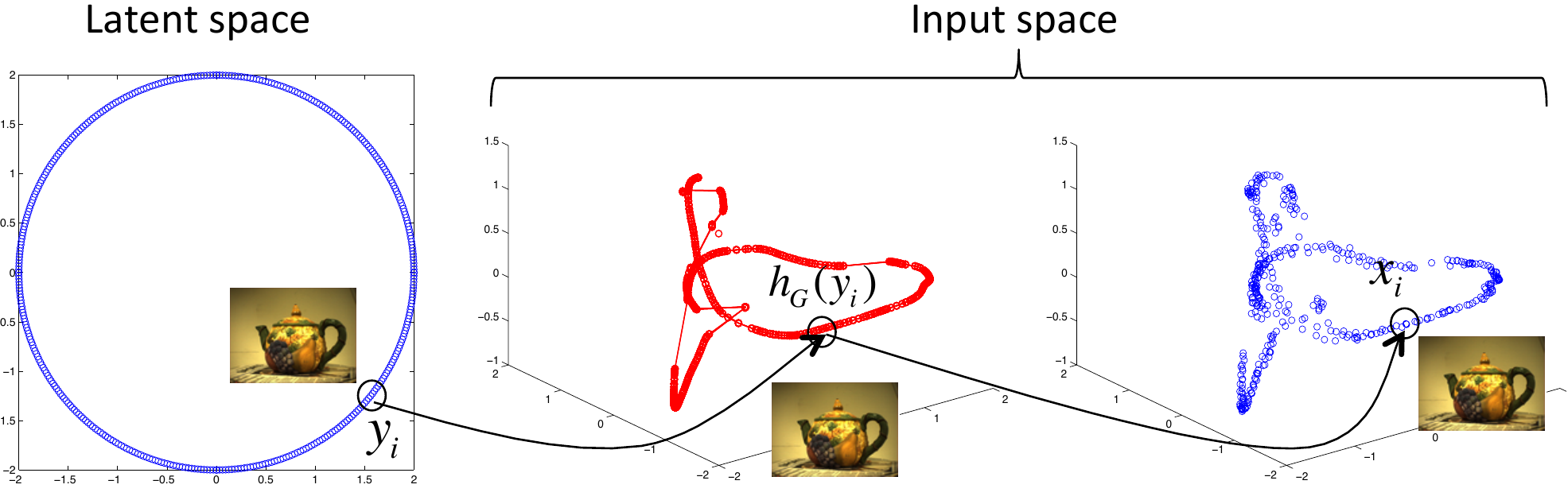} 
	\caption{A cartoon illustrating the proposed formulation on the teapot images. Each circle marker represents one teapot image. Our assumption of the data generation process is that a graph $\G$ exists in the latent space (e.g., a rotation circle in the left subplot), and each image $\bfy_i$ is then mapped to point $h_{\G}(\bfy_i)$ in the input space by maintaining the graph structure through the reversed graph embedding, and finally image $\bfx_i$ is observed conditioned on $h_{\G} (\bfy_i)$ according to certain noise model. } 
	\label{fig:formulation}
\end{figure}

Given an input dataset $\mathbb{x}=\{ \bfx_1, \ldots, \bfx_N \}$, we aim to uncover the underlying graph structure that generates $\mathbb{x}$. Since the input data may be drawn from a noise model, it is improper to learn $\G$ directly from $\mathbb{x}$. To unveil the hidden structure, we assume that (i) the underlying graph can be recovered from the learning model represented in the form of (\ref{op:general-graph}), e.g., the graph is an undirected weighted graph $\G=(\V, \E, \bfW)$ with specific structures; (ii) graph $\G$ satisfies the reversed graph embedding assumption where the cost is defined by $\phi_{i,j}=|| h_{\G}(\bfy_i) - h_{\G}(\bfy_j) ||_2^2$, that is, if two vertices $i$ and $j$ are close on $\G$ with high similarity $w_{i,j}$, points $h_{\G}(\bfy_i)$ and $h_{\G}(\bfy_j)$ should be close to each other; (iii) data points $\{ h_{\G} (\bfy_i) \}_{i=1}^N$ are considered as the denoised points of $\mathbb{x}$ so that the movement of $\bfx_i$ from the denoised point $h_{\G}(\bfy_i)$ can be measured by data grouping properly. 

For clarity, we illustrate the relations of various variables in Fig. \ref{fig:formulation}.  Based on the above assumptions, we propose a novel regularized principal graph learning framework as

\vspace{-0.05in}
\begin{small}
	\begin{align}
	\min_{h_{\G} \in \H, \mathbb{y}} \min_{\bfW \in \W} &~ g(\bfW, \mathbb{x}, \bm{\Phi} ) \label{op:framework}\\
	\textrm{s.t.} &~ ||\bfx_i - h_{\G}(\bfy_i) ||_2 \leq \epsilon_i, \forall i \nonumber \\
	&~ \phi_{i,j}=|| h_{\G}(\bfy_i) - h_{\G}(\bfy_j) ||_2^2, \forall i, j, \nonumber
	\end{align}
\end{small}\noindent
where $\epsilon_i \geq 0$ is a parameter.
It is worth noting that (\ref{op:general-graph}) is the key component of (\ref{op:framework}) for learning a graph given $h_{\G}$ and $\mathbb{y}$, and also contains reversed graph embedding objective (\ref{eq:embed_rep}) as a crucially important element for two special graph learning formulations (\ref{op:spanning-tree}) and (\ref{op:l1-sparse-graph}). In addition, we measure the noise of each input data by the Euclidean distance between the input data point to its denoised counterpart.

Instead of manually setting parameters $\epsilon_i, \forall i$, we assume that the total noise over all input data points is minimal. The optimization problem (\ref{op:framework}) is then reformulated to penalize the total noise
%
where $\gamma >0$ is a tradeoff parameter between graph-based regularizer and the total noises of the input data. We rewrite the above problem as the following optimization problem

\begin{small} \vspace{-0.1in}
	\begin{align*}
	\min_{h_{\G} \in \H, \mathbb{y}} \min_{\bfW \in \W} &~ g(\bfW, \mathbb{x}, \bm{\Phi} ) + \gamma \sum_{i=1}^N || \bfx_i - h_{\G}(\bfy_i) ||_2^2, \\
	\textrm{s.t.} &~ \phi_{i,j}=|| h_{\G}(\bfy_i) - h_{\G}(\bfy_j) ||_2^2, \forall i, j. \nonumber
	\end{align*}
\end{small}\noindent
where the second term in the objective function is equivalent to (\ref{eq:loss-N}).
As we have discussed in Section \ref{sec:data-fitting}, the loss function (\ref{eq:soft-kmeans}) is more flexible than (\ref{eq:loss-N}). 
Consequently, we propose the new regularized principal graph learning framework by simultaneously optimizing the graph weight matrix $\bfW$, latent variables $\mathbb{y}$, projection function $h_{\G} \in \H$, and soft-assignment matrix $\bfP$ as

\begin{small} \vspace{-0.1in}
	\begin{align}
	\min_{h_{\G} \in \H, \mathbb{y}} \min_{\bfW \in \W, \bfP \in \P_r} &~ g(\bfW, \mathbb{x}, \bm{\Phi} ) + \gamma \ell_\bfP( {\mathbb{x}, \mathbb{y}}, h_{\G}) \label{op:framework-P} \\
	\textrm{s.t.} &~ \phi_{i,j}=|| h_{\G}(\bfy_i) - h_{\G}(\bfy_j) ||_2^2, \forall i, j. \nonumber
	\end{align}
\end{small}\noindent
In the following of this section, we propose a simple method for solving problem (\ref{op:framework-P}) and present its convergence and complexity analysis.


\subsection{Optimization Algorithm}
We focus on learning the underlying graph structure from the data. Instead of learning $h_{\G} \in \H$ and $\mathbb{y}$ separately\footnote{ To learn $h_{\G}$, we should define $\H$ properly, e.g., a linear projection function might be possible as studied in our previous work \cite{Mao2005b}.}, we optimize $h_{\G}(\bfy_i)$ as a joint variable $\bfc_i, \forall i$.
Let $\mathbb{c}=\{ \bfc_1, \ldots, \bfc_N\}$ where $\bfc_i = h_{\G}(\bfy_i)$, and $\bfC = [\bfc_1,\ldots,\bfc_N] \in \mathbb{R}^{D \times N}$. In the case of weighted undirected $\ell_1$ graph learning, the optimization problem (\ref{op:framework-P}) with respect to variables $\{ \mathbb{c}, \bfW, \bfP \}$ are reformulated as

\begin{small} \vspace{-0.1in}
	\begin{align}
	\min_{\mathbb{c}, \bfW, \bfP}&~ \sum_{i=1}^N \sum_{j=1}^N w_{i,j} ||\bfc_i - \bfc_j||_2^2 + \lambda \sum_{i=1}^N \Big|\Big| \bfx_i - \sum_{j \not= i} w_{j,i} \bfx_j \Big|\Big|_1 \nonumber\\
	+&~\gamma \sum_{i=1}^N \sum_{j=1}^N p_{i,j} \Big [ || \bfx_i - \bfc_j ||_2^2 + \sigma \log p_{i,j} \Big] \label{op:final-N}\\
	\textrm{s.t.} &~ \bfW \geq 0, \bfW = \bfW^T, diag(\bfW)=0 \nonumber \\
	&~ \sum_{j=1}^N p_{i,j} = 1, p_{i,j} \geq 0, \forall i, j. \nonumber
	\end{align}
\end{small}\noindent
In the case of learning a MST, the feasible space $\W$ is $\W_{\T}$, and the objective function is similar to (\ref{op:final-N}) by removing the second term from the objective function. For simplicity, we do not explicitly show the formulation of learning a spanning tree. 

Problem (\ref{op:final-N}) is a bi-convex optimization problem \cite{Gorski2007}: for fixed $\bfP$ and $\bfW$, optimizing $\mathbb{c}$ is a convex optimization problem; for fixed $\mathbb{c}$, optimizing $\bfP$ and $\bfW$ is also convex.
We propose to solve problem (\ref{op:final-N}) by using the alternate convex search, a minimization method to solve a biconvex problem where the variable set can be divided into disjoint blocks \cite{Gorski2007}. The blocks of variables defined by convex subproblems are solved cyclically by optimizing the variables of one block while fixing the variables of all other blocks. In this way, each convex subproblem can be solved efficiently by using a convex minimization method. Below, we discuss each subproblem in details and the pseudo-code of the proposed method is given in Algorithm \ref{alg:pgl}.

\subsubsection*{Fix $\bfC$ and Solve $\bfP, \bfW$}
Given $\bfC$, Problem (\ref{op:final-N}) with respect to $\bfP $ and $ \bfW $ can be solved independently. 
Due to the decoupled variables of $\bfP$ in rows, we can solve each row independently.
According to Proposition \ref{prop:reformulation}, we have the analytic solution of $\bfP$ given by

\begin{small} \vspace{-0.1in}
	\begin{align}
	p_{i,j} = \frac{\exp \left( - {||\bfx_i - \bfc_j||^2}/{\sigma} \right)  }{ \sum_{j=1}^N \exp \left( - {||\bfx_i - \bfc_j||^2}/{ \sigma}  \right)  }, \forall i, j, \label{eq:final-N-P}
	\end{align}
\end{small}\noindent
For Problem (\ref{op:final-N}) with respect to $\bfW$, we have to solve the following optimization problem

\begin{small} \vspace{-0.1in}
	\begin{align}
	\min_{\bfW} &~ trace( \bm{\Phi}^T \bfW ) + \lambda ||\bfC (\bfI - \bfW)||_1 \label{eq:final-N-W}\\
	\textrm{s.t.} &~\bfW \geq 0, \bfW = \bfW^T, diag(\bfW)=0, \nonumber
	\end{align}
\end{small}\noindent
where $\bm{\Phi} \in \mathbb{R}^{N \times N}$ with the $(i,j)$th element $\phi_{i,j} = ||\bfc_i - \bfc_j||_2^2$.
For Problem (\ref{op:framework-P}) to learn a spanning tree structure, we solve the following problem

\begin{small}\vspace{-0.1in}
	\begin{align}
	\min_{\bfW \in \W_{\T}} trace( \bm{\Phi}^T \bfW ). \label{eq:final-N-W-tree}
	\end{align}
\end{small}\noindent
As discussed in Section \ref{sec:sparse-graph}, Problem (\ref{eq:final-N-W-tree}) can be solved by the Kruskal's algorithm. Problem (\ref{eq:final-N-W}) is a linear programming problem, which can be solved efficiently by off-the-shelf linear programming solver for small- or moderate-sized datasets, such as Mosek \cite{andersen2000mosek}.

\subsubsection*{Fix $\bfP, \bfW$ and Solve $\mathbf{C}$}
Given $ \bfP $ and $\bfW$, Problem (\ref{op:final-N}) with respect to $\bfC$ can be rewritten as

\begin{small} \vspace{-0.1in}
	\begin{align}
	\min_{\bfC}  &~ trace\left( \bfC (2 \bfL + \gamma \bm{\Lambda} ) \bfC^T - 2 \gamma \bfC^T \bfX \bfP \right), \label{op:QP-C}
	\end{align}
\end{small}\noindent
where the graph Laplacian matrix $\bfL = diag(\bfW \bfone) - \bfW$ and diagonal matrix $\bm{\Lambda} = diag( \bfone^T \bfP)$. Problem (\ref{op:QP-C}) is an unconstrained quadratic programming. We have the analytic solution given by

\begin{small} \vspace{-0.1in}
	\begin{align}
	\bfC = \bfX \bfP ( 2 \gamma^{-1} \bfL + \bm{\Lambda} )^{-1},
	\end{align}
\end{small}\noindent
where the inverse of matrix $2 \gamma^{-1} \bfL + \bm{\Lambda} $ always exists since $\bfL$ is a positive semi-definite matrix and diagonal matrix $\bm{\Lambda}$ is always positive definite according to (\ref{eq:final-N-P}).

\begin{algorithm}[!t]
	\caption{Principal Graph Learning}
	\label{alg:pgl}
	\begin{small}
		\begin{algorithmic}[1]
			\STATE \textbf{Input:} Data $\bfX \in \mathbb{R}^{D \times N}$, $\lambda$, $\sigma$, and $\gamma$
			\STATE Initialize $\bfC = \bfX$ 
			\REPEAT 
			\STATE  $\phi_{i,j} = ||\bfc_i - \bfc_j||_2^2, \forall i, j$
			\STATE $p_{i,j} = \frac{\exp \left( - {||\bfx_i - \bfc_j||^2}/{\sigma} \right)  }{ \sum_{j=1}^N \exp \left( - {||\bfx_i - \bfc_j||^2}/{ \sigma}  \right)  }, \forall i, j$
			\STATE Solve the following linear programming problem: \\
			- (\ref{eq:final-N-W}) for a weighted undirected $\ell_1$ graph \\
			- (\ref{eq:final-N-W-tree}) for a spanning tree
			\STATE $\bfC = \bfX \bfP ( 2 \gamma^{-1} \bfL + \bm{\Lambda} )^{-1}$
			\UNTIL{Convergence}
		\end{algorithmic}
	\end{small}\noindent
\end{algorithm}

\subsection{Convergence and Complexity Analysis} \label{sec:convergence}
Since problem (\ref{op:final-N}) is non-convex, there may exist many local optimal solutions. Following the initialization strategy of the mean shift clustering, we initialize $\bfC$ to be the original input data as shown in Algorithm \ref{alg:pgl}. The theoretical convergence analysis of Algorithm \ref{alg:pgl} is presented in the following theorem.

\begin{theorem} \label{th:conv}
	Let $\{\bfW_{\ell},\bfC_{\ell},\bfP_{\ell}\}$ be the solution of Problem (\ref{op:final-N}) in the $\ell$th iteration, and ${\varrho}_{\ell} = {\varrho}(\bfW_{\ell},\bfC_{\ell},\bfP_{\ell})$ be the corresponding objective function value, then we have:
	
	(i) $\{{\varrho}_{\ell}\}$ is monotonically decreasing; 
	
	(ii) Sequences $\{\bfW_{\ell},\bfC_{\ell},\bfP_{\ell}\}$ and $\{{\varrho}_{\ell}\}$ converge.
\end{theorem}

\begin{proof}
Let $\{\bfW_{\ell},\bfC_{\ell},\bfP_{\ell}\}$ be a solution obtained in the $\ell$th iteration. By Algorithm 1, at the $(\ell+1)$th iteration, since each subproblem is solved exactly, we have
	\begin{align*}
	{\varrho}(\bfW_{\ell},\bfC_{\ell},\bfP_{\ell}) &\geq {\varrho}(\bfW_{\ell},\bfC_{\ell},\bfP_{\ell+1}) \geq {\varrho}(\bfW_{\ell+1},\bfC_{\ell},\bfP_{\ell+1}) \\
	&\geq {\varrho}(\bfW_{\ell+1},\bfC_{\ell+1},\bfP_{\ell+1}).
	\end{align*}
%
So sequence $\{{\varrho}_{\ell}\}$ is monotonically decreasing. Furthermore, 
function ${\varrho}(\bfW,\bfC,\bfP)$ is lower-bounded, and then by Monotone Convergence Theorem, there exists ${\varrho}^*\geq -\gamma \sigma N \log N$, such that $\{{\varrho}_{\ell}\}$ converges to ${\varrho}^*$.
%

Next, we prove that the sequence $\{\bfW_{\ell},\bfC_{\ell},\bfP_{\ell}\}$  generated by Algorithm 1 also converges. 
Due to the compactness of feasible sets $\bfW$ and $\bfP$,
we have the sequence $\{ \bfW_{\ell},\bfP_{\ell}\}$ converges to $\{\bfW^*,\bfP^*\}$ as $\ell \rightarrow \infty$.
%
Since $\bfC_{\ell} = \bfX \bfP (2 \gamma^{-1} \bfL + \Lambda )^{-1}$, $\{\bfZ_{\ell} \}$ converges to $\bfC^* = \bfX \bfP^* (2 \gamma^{-1} \bfL^* + \Lambda^*)^{-1}$, where $\bfL^* = diag(\bfW^* \bfone) - \bfW^*$ and $\Lambda^*= diag( \bfone^T \bfP^*)$.
\end{proof}

According to Theorem \ref{th:conv}, we define the stopping criterion of Algorithm \ref{alg:pgl} in terms of the relative increase in the function value compared to the last iteration, and fix it as $10^{-5}$ in all experiments. The empirical convergence results are shown in Section \ref{sec:exconvergence}.

The time complexity of Algorithm \ref{alg:pgl} for learning a tree structure is determined by three individual parts. The first part is the complexity of running  Kruskal's algorithm to construct a minimum spanning tree. It requires ${\mathcal O}(N^2 D)$ for computing a fully connected graph and $ {\mathcal O}(N^2 \log N) $ for finding a spanning tree. The second part is dominated by computing the soft assignments of samples, which has a complexity of ${\mathcal O}(N^2 D)$. The third part is dominated by the inverse of a matrix  of size $N \times N$ that takes ${\mathcal O}(N^3)$ operations and matrix multiplication that takes $\mathcal{O}(DN^2)$ operations. Therefore, the total complexity for each iteration is ${\mathcal O}(N^3 + DN^2)$. For learning an undirected weighted $\ell_1$ graph by solving Problem (\ref{eq:final-N-W}), the only difference is the complexity of linear programming, which can be solved efficiently by Mosek \cite{andersen2000mosek}. In Section \ref{sec:large}, we will extend the proposed algorithm for dealing with large-scale data.

\section{Regularized Principal Graph Learning on  Large-Scale Data} \label{sec:large}

In order to reduce the high computational complexity of Algorithm \ref{alg:pgl} for large-scale data, we propose to incorporate two strategies into the proposed model for fast learning by using landmarks and side information.

\subsection{Graph Learning with Side Information}

Instance-level constraints are useful to express a priori knowledge about which instances should or should not be grouped together, which have been successfully applied to many learning tasks, such as clustering \cite{Wagstaff2001}, metric learning \cite{xing2002distance}, and kernel learning \cite{zhuang2011family}. In the area of clustering, the most prevalent
form of advice are conjunctions of pairwise instance level-constraints of the
form must-link (ML) and cannot-link (CL) which state that pairs of instances
should be in the same or different clusters respectively \cite{Wagstaff2001}. Given a set of points to cluster and a set of constraints, the aim of clustering with constraints
is to use the constraints to improve the clustering results.

We consider certain structure-specific side information for guidance of learning the similarity matrix $\bfW$ and computational reduction instead of optimizing the full matrix. For this purpose, we take into account the CL constraints. Let $\N_i$ be a set of data points which might be linked to data point $\bfx_i$. On the contrary, data points that are not in $\N_i$ will belong to CL set. By incorporating these CL side information into the proposed framework, we can come out the following optimization problem for learning $\ell_1$ graph representation given by

\begin{small}\vspace{-0.1in}
	\begin{align}
	\min_{\bfW}&~ \sum_{i=1}^N \sum_{j \in \N_i} w_{i,j} \phi_{i,j} + \lambda \sum_{i=1}^N \Big|\Big| \bfx_i - \sum_{j \in \N_i} w_{j,i} \bfx_j \Big|\Big|_1 \\
	\textrm{s.t.} &~ w_{i,j} \geq 0, w_{i,j} = w_{j,i}, \forall i, j \nonumber \\
	&~ w_{i,j}=0, \forall i, j \not \in \N_i \nonumber.
	\end{align}
\end{small}\noindent
which can be equivalently transformed into linear programming optimization problem
and can be solved accurately for large-scale data by using the off-the-shelf toolbox.

\subsection{Landmark-based Regularized Principal Graph} \label{sec:landmark}

In order to handle large-scale data, we extend the proposed regularized principal graph learning framework by using a landmark-based technique.
Landmarks have been widely used in clustering methods to deal with large-scale data \cite{Fowikes2004, Chen2011}. Random landmark selection method is widely used for spectral clustering \cite{Fowikes2004, Chen2011}. 
It is well known that selecting landmarks can further improve clustering performance with a certain adaptive strategy, such as $k$-means \cite{Yan2009}, a variety of fixed and adaptive sampling schemes, and a family of ensemble-based sampling algorithms \cite{Kumar2012}. 

Our regularized principal graph learning framework can inherently take the landmark data points into account through the data grouping. Taking the weighted undirected $\ell_1$ graph based framework as an example, 
the optimization problem by considering landmarks is given by 

\begin{small} \vspace{-0.1in}
	\begin{align}
	\min_{\mathbb{c}, \bfW, \bfP}&~ \sum_{k=1}^K \sum_{k'=1}^K w_{i,j} ||\bfc_k - \bfc_{k'}||_2^2 + \lambda \sum_{k=1}^K \Big|\Big| \bfz_k - \sum_{k' \not= K} w_{k',k} \bfz_{k'} \Big|\Big|_1 \nonumber\\
	+&~\gamma \sum_{i=1}^N \sum_{k=1}^K p_{i,k} \Big [ || \bfx_i - \bfc_k ||_2^2 + \sigma \log p_{i,k} \Big] \label{op:final-K}\\
	\textrm{s.t.} &~ \bfW \geq 0, \bfW = \bfW^T, diag(\bfW)=0 \nonumber \\
	&~ \sum_{k=1}^K p_{i,k} = 1, p_{i,k} \geq 0, \forall i=1,\ldots,N, k =1,\ldots,K. \nonumber
	\end{align}
\end{small}\noindent
where $\mathbb{z}=\{ \bfz_1,\ldots, \bfz_K \}$ is a set of landmarks of size $K$.
Moreover, our proposed methods can automatically adjust centroids of landmarks during the optimization procedure. Due to the non-convexity of above objective function, we have to properly initialize these landmarks. 
As shown in the literature \cite{Yan2009, Chen2011}, $k$-means can generally obtain better clustering performance for spectral clustering methods. In this paper, we follow this idea and use centroids obtained from $k$-means to initialize $\bfZ$. The pseudo-code is shown in Algorithm \ref{alg:pgl-large}. By following the same analysis procedure as Section \ref{sec:convergence}, the computational complexity of Algorithm \ref{alg:pgl-large} is  $\mathcal{O}(K^3 + D K N + D K^2 )$ which is significantly smaller than ${\mathcal O}(N^3 + DN^2)$ of Algorithm \ref{alg:pgl} if $K \ll N$. Hence, Algorithm \ref{alg:pgl-large} is practical for large-scale data with a large $N$.

\begin{algorithm}[!t]
	\caption{Large-Scale Principal Graph Learning}
	\label{alg:pgl-large}
	\begin{small}
		\begin{algorithmic}[1]
			\STATE \textbf{Input:} Data $\bfX \in \mathbb{R}^{D \times N}$, $\lambda$, $\sigma$, $\gamma$ and $K \ll N$
			\STATE Obtain $\bfZ \in \mathbb{R}^{D \times K}$ by calling the $K$-means method
			\STATE Initialize $\bfC = \bfZ$ 
			\REPEAT 
			\STATE  $\phi_{k,k'} = ||\bfc_k - \bfc_{k'}||_2^2, \forall k, k'$
			\STATE $p_{i,k} = \frac{\exp \left( - {||\bfx_i - \bfc_k||^2}/{\sigma} \right)  }{ \sum_{k=1}^K \exp \left( - {||\bfx_i - \bfc_k||^2}/{ \sigma}  \right)  }, \forall i, k$
			\STATE Solve the following linear programming problem: \\
			- (\ref{eq:final-N-W}) for a weighted undirected $\ell_1$ graph \\
			- (\ref{eq:final-N-W-tree}) for a spanning tree
			\STATE $\bfC = \bfX \bfP ( 2 \gamma^{-1} \bfL + \bm{\Lambda} )^{-1}$
			\UNTIL{Convergence}
		\end{algorithmic}
	\end{small}\noindent
\end{algorithm}

\section{Experiments} \label{sec:experiments}
In this section, we conduct extensive experiments to evaluate the proposed models for learning either spanning trees or weighted undirected $\ell_1$ graphs on various synthetic datasets and six real world applications.

\subsection{Convergence and Sensitivity Analysis} \label{sec:exconvergence}

We perform a convergence analysis of Algorithm \ref{alg:pgl} using a synthetic tree dataset. Fig. \ref{fig:covergence} shows the empirical convergence results obtained by learning two different graph structures as well as their intermediate results. 
The top panel of Fig. \ref{fig:covergence} shows that the empirical convergence results by illustrating the relative difference of the objective function value in terms of the number of iterations in Algorithm \ref{alg:pgl}. 
We observe that the proposed algorithm converges in less than $20$ iterations. This is consistent with the result of the theoretical analysis in Section \ref{sec:convergence}. From the bottom panel of Fig. \ref{fig:covergence}, we can see that when Algorithm \ref{alg:pgl} continues with more iterations, the tree structure becomes smoother. This empirically verifies the intuition of the reverse graph embedding.

We then perform a parameter sensitivity analysis by using the distorted S-shape dataset to demonstrate how the algorithm behaves with respect to parameters $\sigma$, $\gamma$ and $\lambda$. As $\lambda$ only appears in Algorithm \ref{alg:pgl} for learning an $\ell_1$ graph, we investigate the sensitivity of parameters in terms of learning an $\ell_1$ graph. Results obtained from learning an $\ell_1$ graph can be similarly applied to Algorithm \ref{alg:pgl} for learning a spanning tree. Fig. \ref{fig:sensitivity} shows the $\ell_1$ graphs constructed by varying one parameter and fixing the others. We have the following observations. First, according to Fig. \ref{fig:sensitivity} (a), it is clear that the smaller $\gamma$ is, the shorter the length of the S-shape curve is. In contrast, the larger $\gamma$ is, the more faithful the curve passes through the middle of the data. Therefore, $\gamma$ is an important parameter that controls the trade-off between the curve fitting error and the length of a principal graph. Second, as shown in Fig. \ref{fig:sensitivity}(b), the graph structure becomes smoother when increasing $\sigma$. This is also stated in Propositions \ref{prop:mean-shift} and \ref{prop:reformulation} that explore the relationships between the proposed formulation and the mean-shift clustering. In other words, a large $\sigma$ encourages more data points to merge together. The data movement represented by variable matrix $\bfC$ in Algorithm \ref{alg:pgl} is also restricted on a graph and the length of the graph should be minimized so that data points in $\bfC$ will be smoother so as to reach the goal. The choice of bandwidth parameter in Algorithm \ref{alg:pgl} is then similar to that in the mean-shift clustering. As suggested in \cite{carreira2015review}, it is best to explore a range of bandwidths instead of estimating one from data by minimizing a loss function or more heuristic rules since clustering is by nature exploratory, so is the learning structure in this paper. Third, the larger $\lambda$ is, more edges the learned graph contains as shown in Figs. \ref{fig:sensitivity}(c)-(f). This is the reason why the noise term becomes large and less data points are required to represent the current data point if $\lambda$ is a large value.

\begin{figure}
\centering
\includegraphics{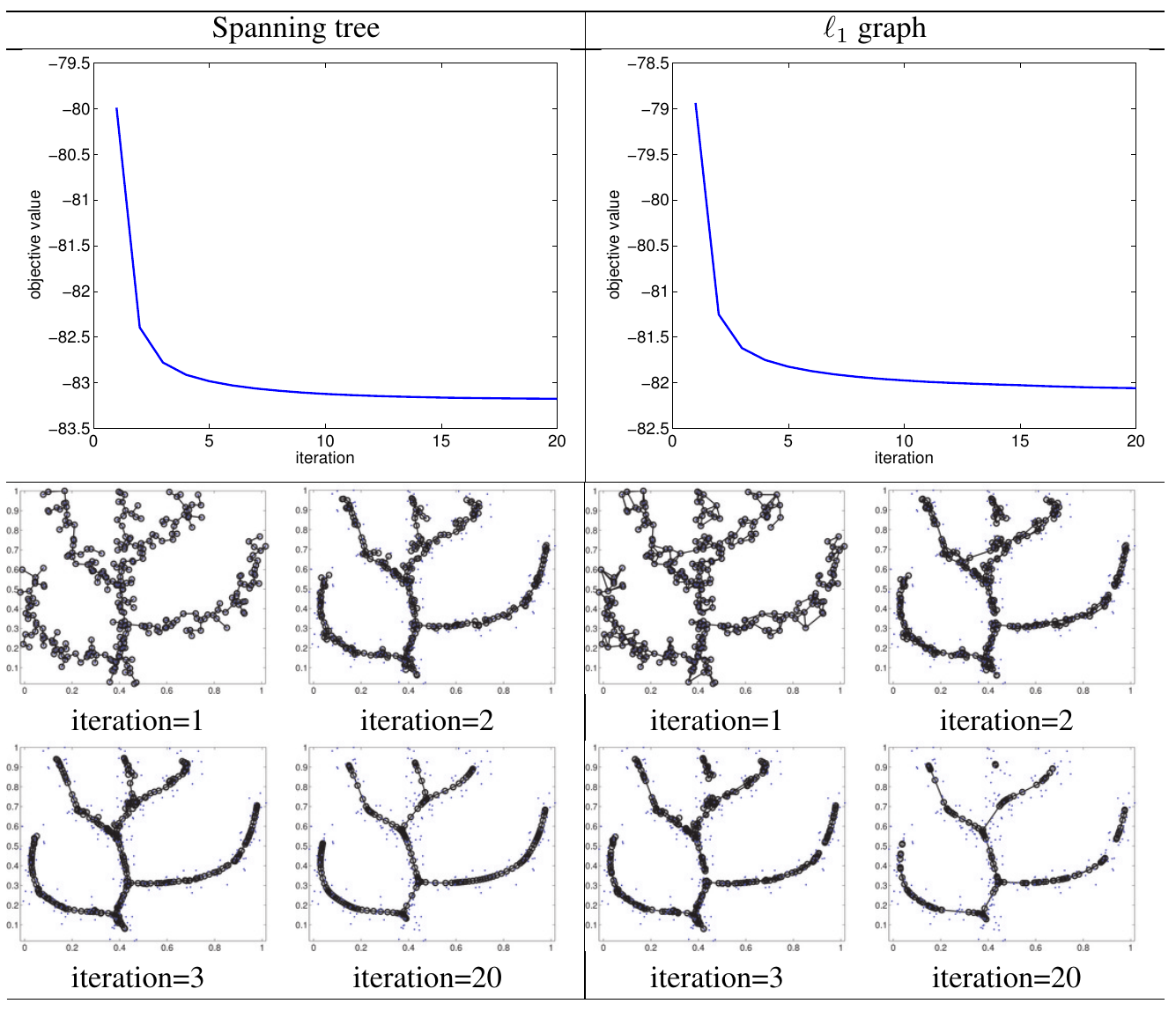}
\caption{Convergence analyses and intermediate results of the proposed method performed on the Tree dataset by learning two different graph structures.} \label{fig:covergence}
\end{figure}

%
%

Due to the exploratory nature of learning a graph structure from data and the lack of evaluation metrics for graph structure learning in an unsupervised setting, the automation of setting parameters by using the provided data only is generally difficult. However, the aforementioned parameter sensitivity analysis suggests a pragmatic way to tune parameters for learning an expected graph structure. The recommended procedure is given as follows: first, setting $\sigma$ to be a small value, and then tuning $\gamma$ from large value to small so as to fit the data properly, finally increasing $\sigma$ to obtain a reasonable structure. In order to learn an $\ell_1$ graph from data, the initial graph structure should be constructed by solving Problem (\ref{eq:final-N-W}) using a proper $\lambda$ such that the weight matrix $\bfW$ captures the key structure of the input data. And then, the above recommended procedure for tuning $\sigma$ and $\gamma$ is applied. For tuning $\sigma$ and $\gamma$ automatically from data, we have studied in our previous work \cite{Mao2015} by using leave-one-out-maximum likelihood criterion described in \cite{Ozertem2011} and gap statistics \cite{Witten2010} for $\sigma$ and $\gamma$ respectively. As discussed in \cite{carreira2015review}, it is natural to tune parameter $\sigma$ in a different range.
In the following experiments, we will take the above tuning strategy for setting all parameters in a large range, and all parameters used will be reported for the reproducibility of the experiments.


\begin{figure}
\centering
\includegraphics{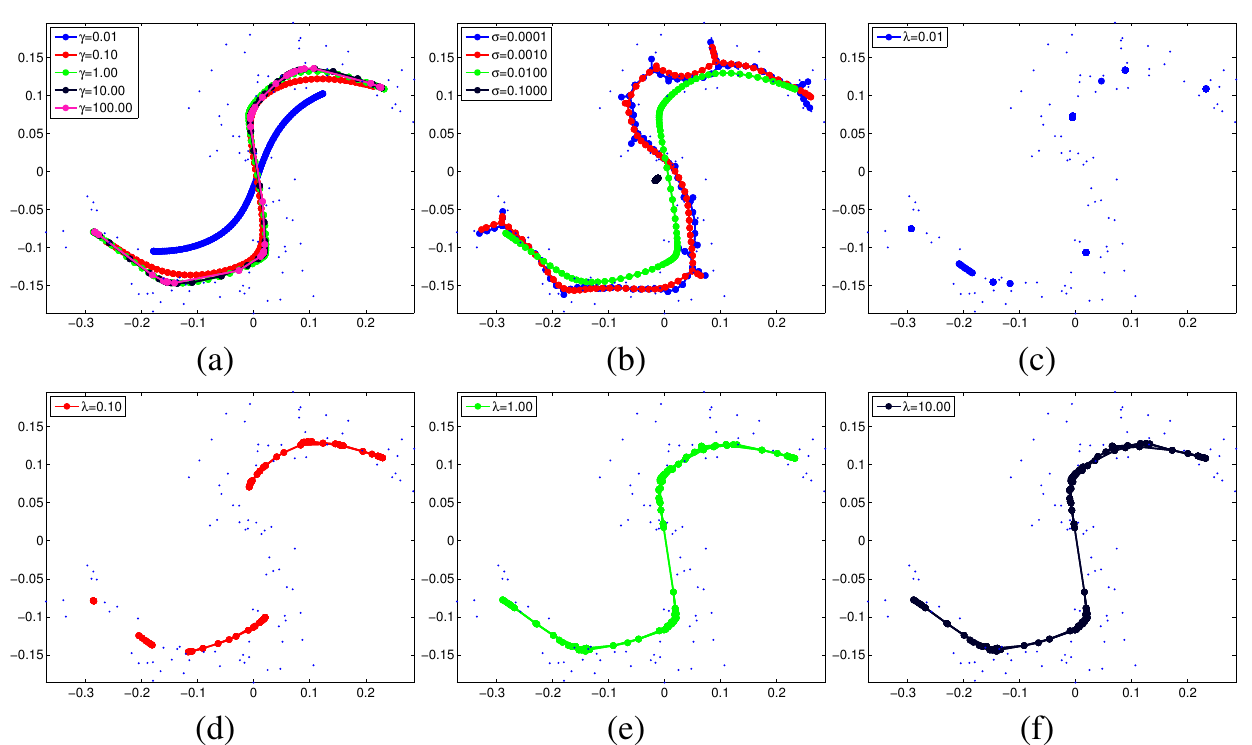}
\caption{Results of parameter sensitivity analysis of Algorithm \ref{alg:pgl} by learning $\ell_1$ graph performed on the distorted S-shape dataset by varying different parameters: (a) varying $\gamma$; (b) varying $\sigma$; (c)-(f) varying $\lambda$. } \label{fig:sensitivity}
\end{figure}


\subsection{Synthetic Data}

\begin{figure}
\centering
\includegraphics{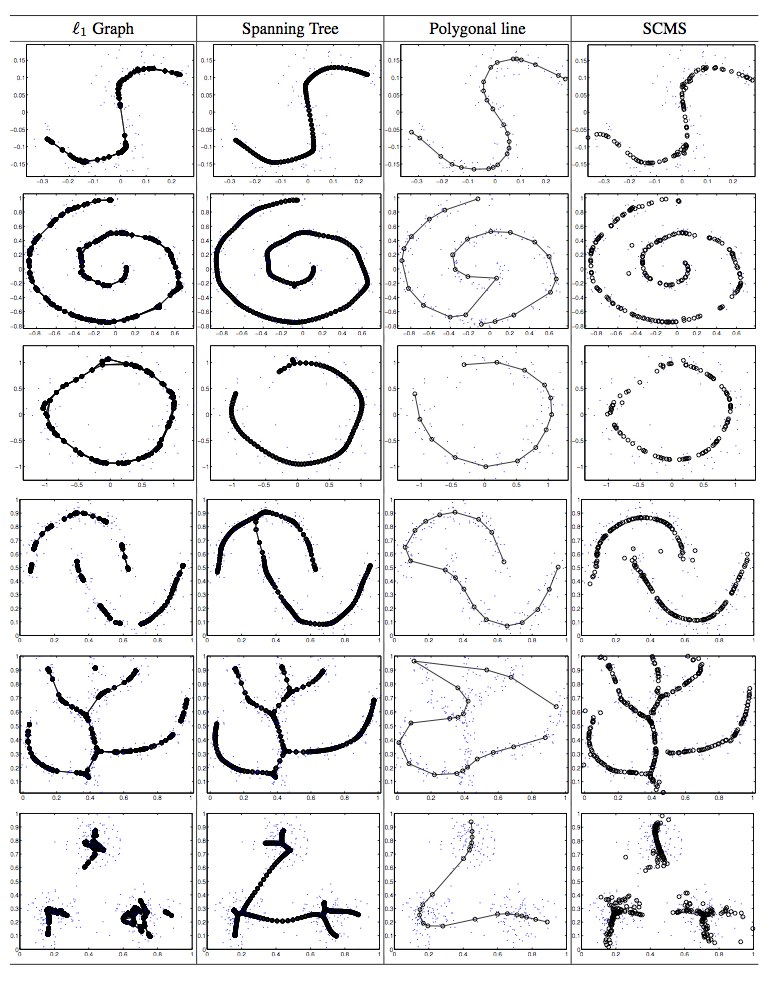}
	\caption{Results of four principal curve methods performed on six synthetic datasets containing various situations including curves, loops, self-intersections, and multiple disconnected components. The first and second columns show the results of our proposed methods for learning either an $\ell_1$ graph or a spanning tree. 	
		The third and forth columns report the results generated by the polygonal line method and SCMS, respectively.} \label{fig:curves}
\end{figure}

We evaluate the performance of Algorithm \ref{alg:pgl} for learning either a spanning tree or an $\ell_1$ graph by comparing with the polygonal line method \cite{Kegl2000} and SCMS \cite{Ozertem2011} on six synthetic datasets. Among them, the first two datasets are also used in \cite{Kegl2000, Ozertem2011}. In the case of learning an $\ell_1$ graph, we incorporate neighbors of each data point as the side information, i.e., data points not in the neighborhood of a data point are considered as a set of cannot-links for the data point. In this paper, we take $nn$-nearest neighbor as a showcase for side-information.  The experiments are conducted in two settings. The first setting is to evaluate the four methods for learning curves, while the second setting is to investigate other  structures of the datasets, including loops, self-intersection and disconnected components. For all experiments in this section, the parameters of compared methods are set as shown in Table \ref{tab:toy-params}. It is worth noting that the tuning of parameters in the proposed method is relatively easier than tuning $\sigma$ in SCMS, which needs finer tunning in order to achieve comparable results.

\begin{table}
	\caption{Parameters used in the proposed algorithm for synthetic datasets. The proposed method for learning an $\ell_1$ graph and a spanning tree share the same set of parameters, while SCMS are tuned in fine-grid to achieve reasonable principal points. } \label{tab:toy-params}
	\centering
	\begin{tabular}{lcccccc}
		\hline
		Dataset	 & $N$ & $\sigma$ & $\gamma$ & $\lambda$ & nn & SCMS $\sigma$ \\
		\hline
		Distorted S-shape & 100 & 0.01 & 0.5 & 1.0 & 5 &0.07\\
		Spiral & 200 & 0.01 & 0.5 &1.0 & 10 &0.1\\
		Circle & 100  & 0.1 & 0.5  &1.0 & 10 & 0.4\\	 
		Two-moon &200& 0.01 & 3 & 0.1 & 5 &0.1\\
		Tree &300& 0.01 & 10 & 1.0 & 5 & 0.05\\
		Three-clusters &300& 0.01 & 0.5 & 0.1 & 5 & 0.1\\
		\hline
	\end{tabular}
\end{table}

The first two rows of Fig. \ref{fig:curves} show the results from the principal curve learning setting. We have the following observations: 1) The proposed methods can obtain more smoothing curves than the other tested methods. 2) The polygonal line method fails on the Spiral data. We also see that SCMS cannot obtain a curve structure because many projected points do not have ordering information, and some points are scattered as shown in the Spiral data. This leaves a non-trivial problem to learn the underlying structure by using SCMS. Our proposed method does not have these problems. In the following experiments, we will demonstrate the importance of learning an explicit structure in various real-world applications.
The last four rows of Fig. \ref{fig:curves} show the results obtained from the datasets containing loop, self-intersection, and disconnected components. The polygonal line method fails on all four datasets due to the improper assumption of principal curves. We should point out that the tree structure is naturally used for learning a tree structure, so the proposed method for learning a tree structure outperforms the one proposed for learning an $\ell_1$ graph as shown in the results on the tree data. However, the latter outperforms the former for datasets with multiple disconnected components and loops.
All results in the second setting are consistent with those obtained in the first setting.

The above results suggest that 1) our methods can obtain results at least as good as the baselines with the properly tuned parameters using the recommended strategy, and outperform the polygonal line method that may fail on some datasets; 2) our methods can handle more complicated data structures than the polygonal line method, and can provide graph structures that cannot be easily obtained by SCMS.

\subsection{Rotation of Teapot Images}

A collection of $400$ teapot images from \cite{Weinberger2006} are used\footnote{http://www.cc.gatech.edu/$\sim$lsong/data/teapotdata.zip}.  These images were taken successively as a teapot was rotated $360^{\circ}$. Our goal is to construct a principal curve that organizes the 400 images.  Each image consists of $76 \times 101$ pixels and is represented as a vector. The data in each dimension is normalized to have zero mean and unit standard deviation. Similar to \cite{Song2007}, a kernel matrix ${\bf X}$ is generated where $X(i,j) =  \exp(- ||\bfx_i - \bfx_j||^2 /D)$.

We run our proposed method using the kernel matrix as the input. We set $\sigma=0.5, \gamma=100, \lambda=1$ and $nn=10$. We first perform PCA on the kernel matrix and project it to a $36$-dimensional space that keeps $95\%$ of total energy. The experimental results of the proposed method for learning a spanning tree and an $\ell_1$ graph are shown in Fig. \ref{fig:teapot}. The principal curves (Fig. \ref{fig:teapot}(a) and \ref{fig:teapot}(c)) are shown in terms of the first $3$ columns of the learned projection matrix $\bfW$ as the coordinates where each dot $\bfy_i$ represents the $i$th image and the sampled images at intervals of $30$ are plotted for the purpose of visualization. Fig. \ref{fig:teapot}(b) and \ref{fig:teapot}(d) show the linear chain dependency among the teapot images following the consecutive rotation process. We can see that the curve generated by our method is in agreement with the rotating process of the $400$ consecutive teapot images. For this data, $\ell_1$ graph is more reasonable than spanning tree since the underlying manifold forms a loop. This is clearly demonstrated in Fig. \ref{fig:teapot}.


\begin{figure}
	\centering

\includegraphics{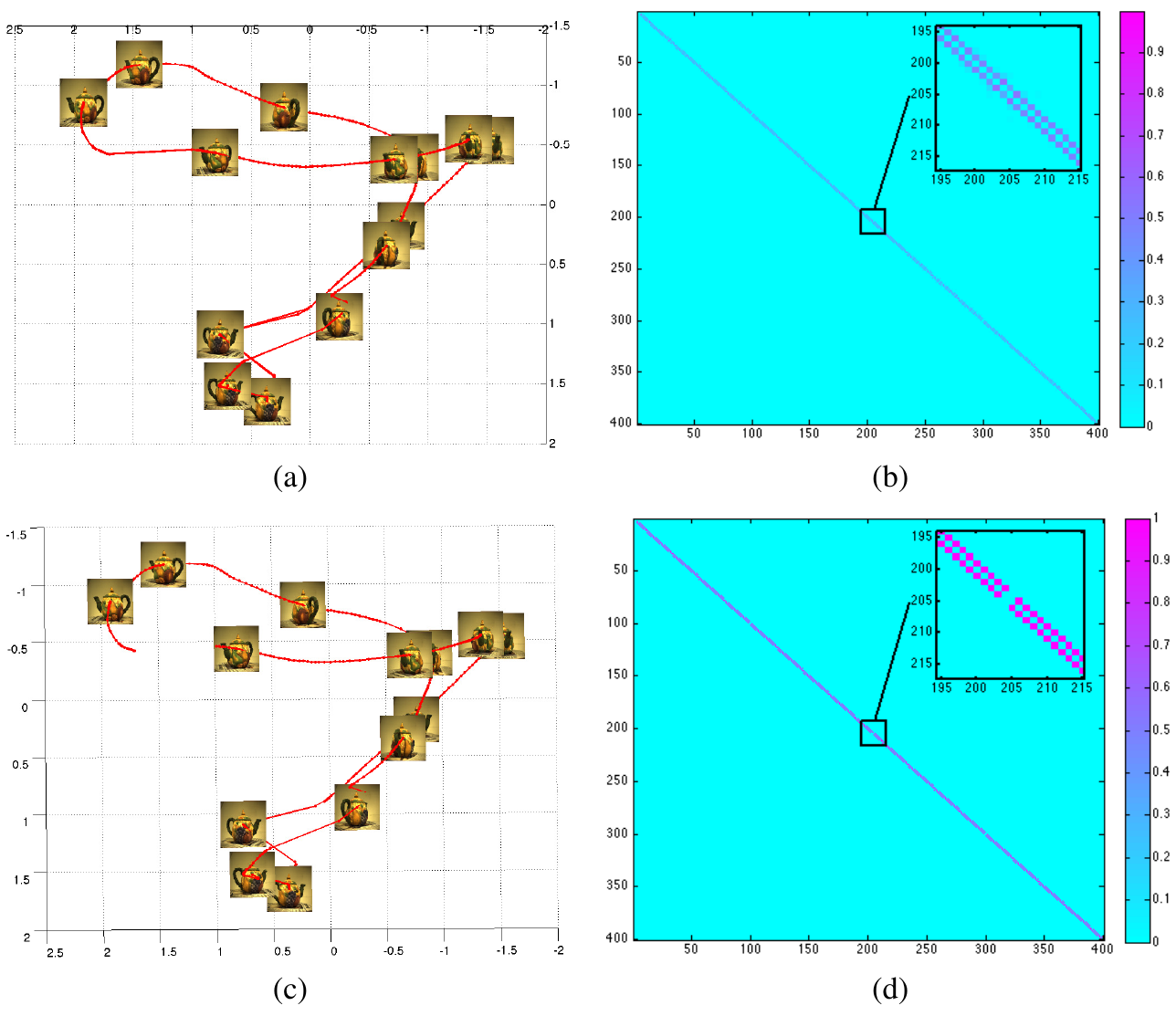} 

	\caption{Experimental results of the proposed method applied to Teapot images. (a) principal circle generated by the proposed method for learning a $\ell_1$ graph. Each dot represents one teapot image. Images following the principal curve are plotted at intervals of $30$ for visualization. (b) The adjacency matrix of the circle follows the ordering of the $400$ consecutive teapot images with $360^\circ$ rotation. (c)-(d) the principle curve and the adjacency matrix are obtained by the proposed method for learning a spanning tree. } \label{fig:teapot}
\end{figure}

A similar result is also recovered by CLUHSIC, which assumes that the label kernel matrix is a ring structure \cite{Song2007}. However, there are two main differences. First, the principal curves generated by our methods are much smoother than that obtained by CLUHSIC (see Figure 4 in \cite{Song2007}). Second, our method learns the adjacency matrix from the given dataset, but CLUHSIC requires a label matrix as {\em a priori}. A two-dimensional representation of the same set of teapot images is given in \cite{Weinberger2006},  where Maximum Variance Unfolding (MVU) is used that arranges the images in a circle (see Figure 3 in \cite{Weinberger2006}). We attempted to run MVU by keeping $95\%$ energy, i.e., $d=36$. However, storage allocation fails due to the large memory requirement of solving a semi-definite programming problem in MVU. Hence, MVU cannot be applied to learn a relatively large intrinsic dimensionality. However, our method does not suffer from this issue.

\subsection{Hierarchical Tree of Facial Images}

Facial expression data\footnote{http://www.cc.gatech.edu/$\sim$lsong/data/facedata.zip} is used for hierarchical clustering, which takes into account both the identities of individuals and the emotion being expressed \cite{Song2007}. This data contains $185$ face images ($308 \times 217$ RGB pixels) with three types of facial expressions (NE: neutral, HA: happy, SO: shock) taken from three subjects (CH, AR, LE) in an alternating order, with around $20$ repetitions each. Eyes of these facial images have  been aligned, and the average pixel intensities have been adjusted. As with the teapot data, each image is represented as a vector, and is normalized in each dimension to have zero mean and unit standard deviation. 

A kernel matrix is used as the input to Algorithm \ref{alg:pgl} for learning a spanning tree in the setting of $\sigma=0.01 $ and $ \gamma=1$. The experimental results are shown in Fig. \ref{fig:face}. We can clearly see  that three subjects are connected through different branches of a tree.  If we take the black circle in Fig. \ref{fig:face}(a) as the root of a hierarchy, the tree forms a two-level hierarchical structure.  As shown in Fig. \ref{fig:face}(b), the three facial expressions from three subjects are also clearly separated. A similar two-level hierarchy is also recovered by CLUHSIC (Figure 3(b) in \cite{Song2007}). However,  the advantages of using the proposed method discussed above for the teapot images are also applied here.
In addition, we can observe more detailed information from the tree structure. For example, LE@SO is the junction to other two subjects, i.e., AR@SO and CH@SO, which can be observed from the $9$th row of the adjacency matrix as shown in  Fig. \ref{fig:face}(b). This observation suggests that the shock is the most similar facial expression among three subjects. However, CLUHSIC is not able to obtain this information. 


\begin{figure} 
	\centering
		\includegraphics{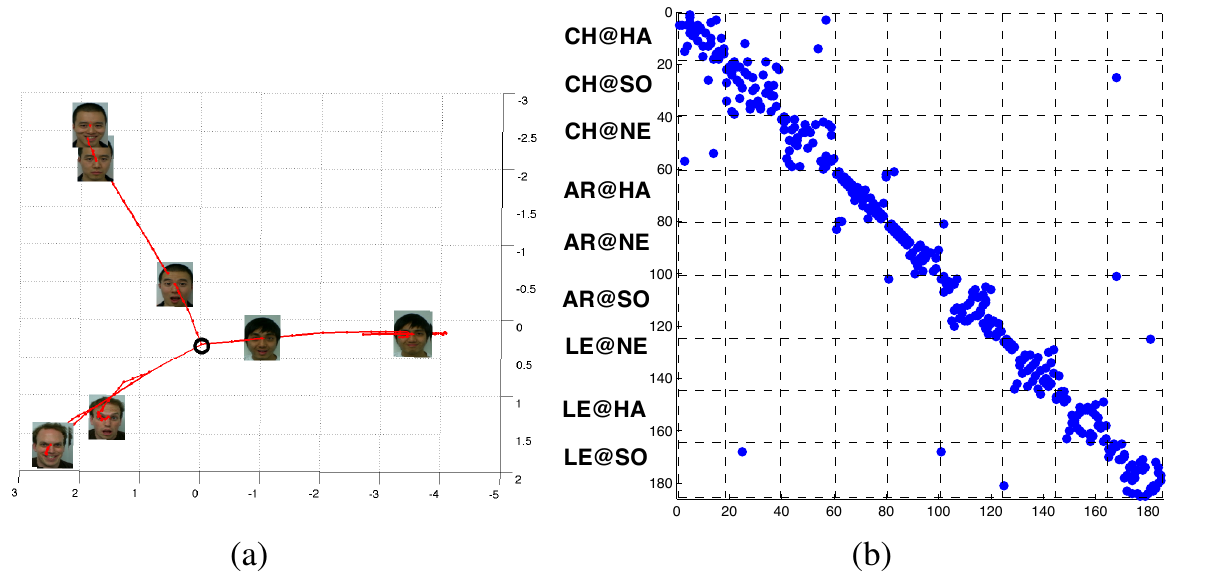}
	\caption{Experimental results of our proposed method for learning a spanning tree performed on facial expression images. (a) the generated hierarchical tree. Each dot represents one face image. Images of three types of facial expressions from three subjects are plotted for visualization. The black circle can be considered as the root of the hierarchical structure for achieving two layers hierarchy over nine subjects; (b)  The adjacency matrix of the tree on nine blocks indicates that each block corresponds to one facial expression of one subject. } \label{fig:face}
\end{figure}

\subsection{Breast Cancer Progression Path}


\begin{figure}
\centering
		\includegraphics{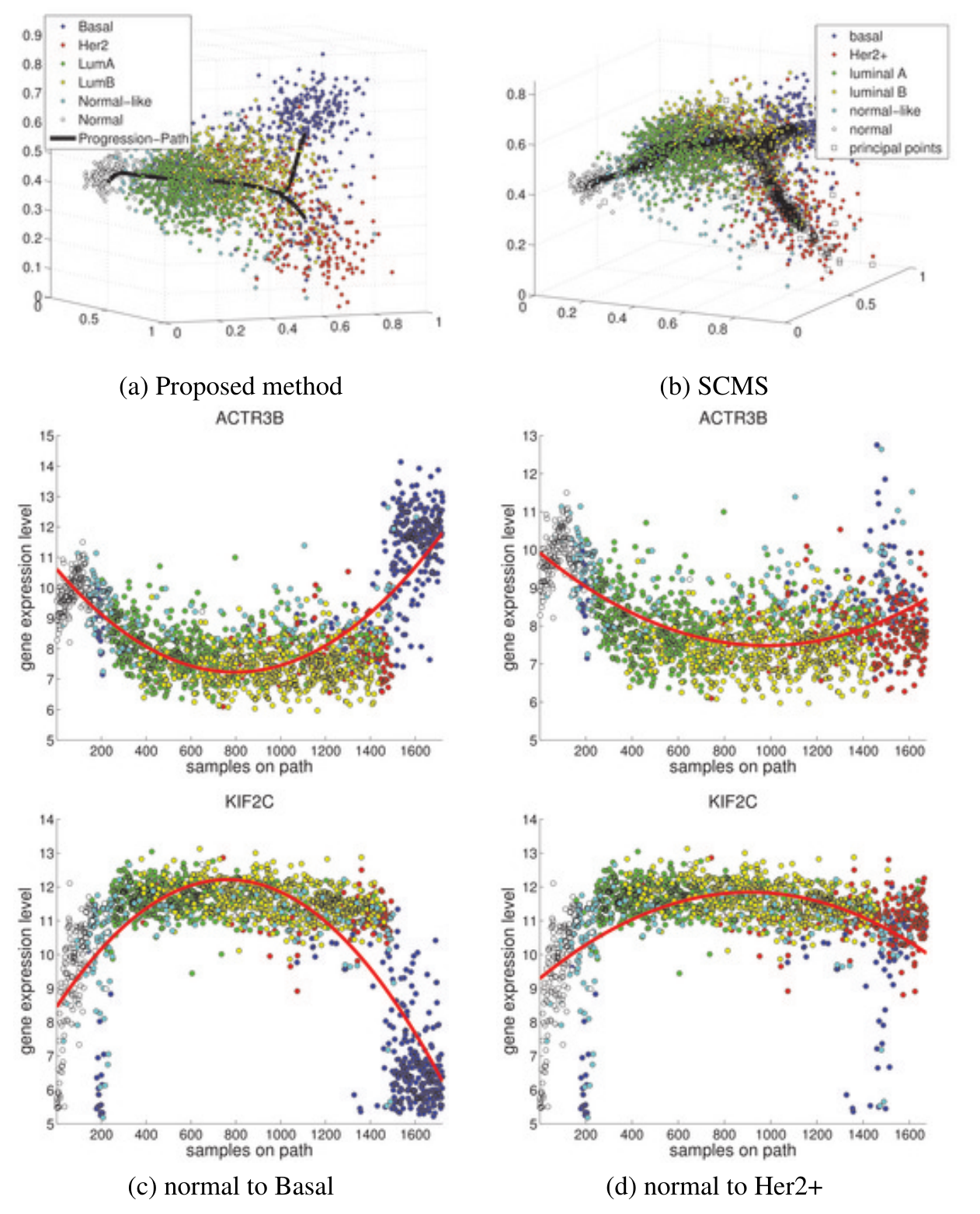}  
		\caption{Progression path of breast cancer data. (a) the tree structure learned by the proposed method for learning a spanning tree; (b) the principal points learned by SCMS; (c) the fitted curves of gene expression of one gene with respect to all samples on the path from normal to Basel by using quadratic function; (d) the fitted curves of gene expression of one gene with respect to all samples on the path from normal to Her2+ by using quadratic function. } \label{fig:cancer}
\end{figure}

In this section, we demonstrate the utility of the proposed method for unveiling the cancer progression path from microarray samples. To this end, we interrogated a large-scale, publicly available breast cancer dataset \cite{Shah2012}. The dataset contains the expression levels of over $25,000$ gene transcripts and the data on gene copy number for over $30,000$ genes. The molecular data was obtained from 144 normal breast tissue samples and $1,989$ tumor samples, and survival data of patients with disease are available.  A total of $1,140$ genes associated with cancer progression were selected using a non-linear regression method and clustering analysis identified genetically homogeneous tumor groups \cite{Sun2014}. To visually demonstrate the cancer progression path, the learned principal points are projected into three-dimensional space via the first three principal components of the data, and the tree structure retains as that learned in the high-dimensional space. As the polygonal line method does not work well on this data, we will not report its result.

We apply the proposed method to visualize the underlying data structure, which in this case represents the progression path of breast cancer towards malignancy (Fig. \ref{fig:cancer}). The revealed data manifold indicates a linear bifurcating progression path for breast cancer progression, confirming previous studies \cite{Sun2014}. Solely to assist with visualization (Fig. \ref{fig:cancer}(a)) we color-coded the tumor samples according to the PAM50 molecular subtype system \cite{Parker2009}.  This system categorizes tumor tissue samples into normal-like, luminal A, luminal B, Her2+, and basal subtypes based on the expression of $50$ gene transcripts. We used the $144$ normal breast tissue samples in the dataset as a baseline.  It is evident that the linear progression path tracks logically from normal tissue samples through luminal subtypes and then splits to either the Her2+ tumor samples or the basal tumor samples. The latter two subtypes are known to be the most aggressive breast tumor types. 

As data from human tissue molecular profiling accumulates, advanced computational analyses that can provide insights into disease progression become increasingly important. The implementation of the proposed method will facilitate the derivation of interactive progression models that can enable the identification of the key molecular events associated with carcinogenesis and malignancy. To further demonstrate the importance of learning an explicit tree structure, we extract two longest path by traversing the learned tree from one leaf node labeled as normal sample to the other leaf node labeled as either Basal or Her2+, respectively. The path from normal to Basal consists of $1722$ samples, while the path from normal to Her2+ consists of $1673$ samples. For $50$ important genes used in PAM50 subtypes, we plot the variation of gene expression in terms of the samples over a given path. Fig. \ref{fig:cancer}(c) and (d) show the results of selected genes (ACTR3B and KIF2C) on two different paths and demonstrate different trend of gene expression variation. These genes might be the causes for cancer progression evolved from luminal B to either Basal or Her2+. For example, if the expression level of gene ACTR3B increases suddenly, the patient might go to the cancer state of Basal.
However, the principal points returned by SCMS as shown in Fig. \ref{fig:cancer}(b) does not have a clear progression structure, so it cannot obtain the similar implications as the proposed method.

\subsection{Skeletons of Optical Characters}

\begin{figure*}[!htb]
\centering
		\includegraphics{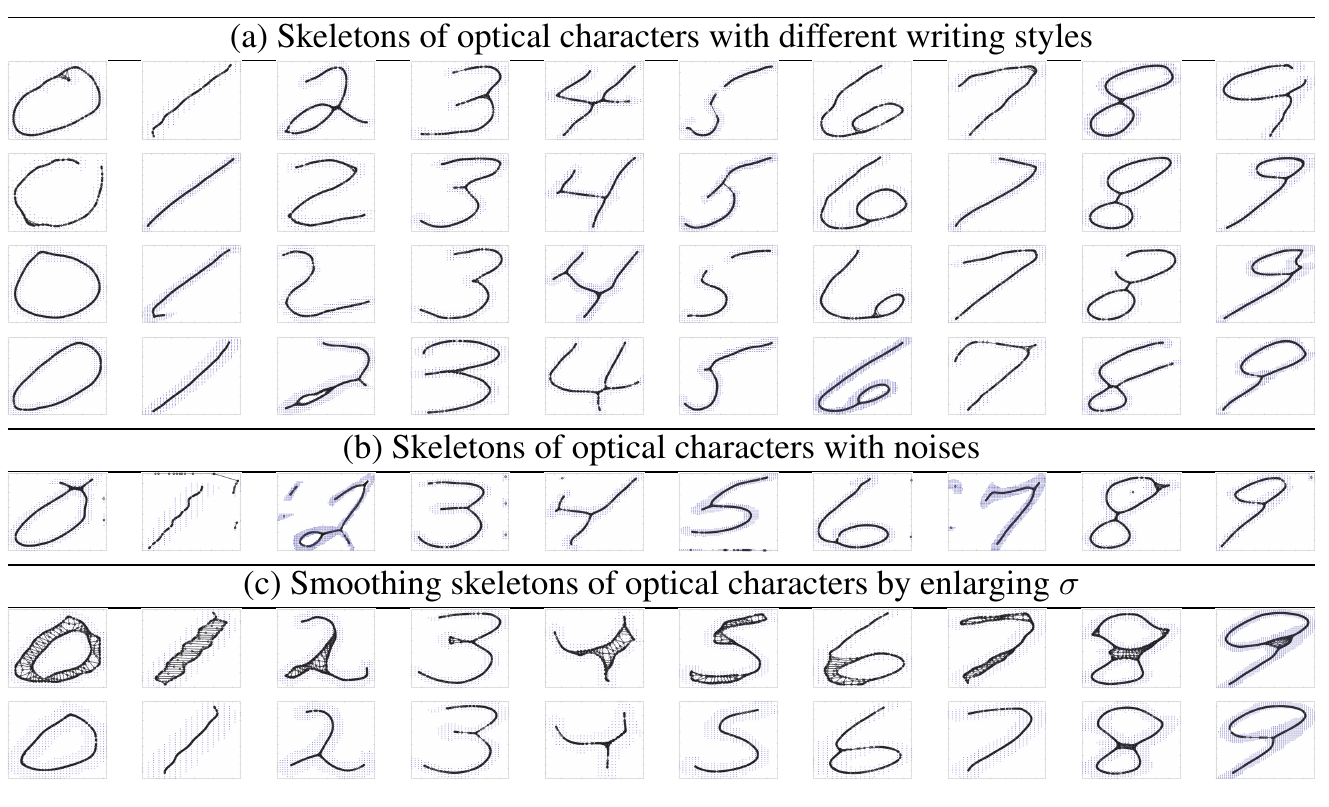} 
	\caption{Results of the proposed RPG-$\ell_1$ in finding the smoothing skeletons of optical characters: (a) the learned smoothing skeletons of optical characters with different writing styles; (b) the learned smoothing skeletons of optical characters with noises; (c) the adjusted smoothing skeletons of optical characters by enlarging $\sigma$ properly.} \label{fig:ocr} \vspace{-0.1in}
\end{figure*}


\begin{figure}
	\centering
		\includegraphics{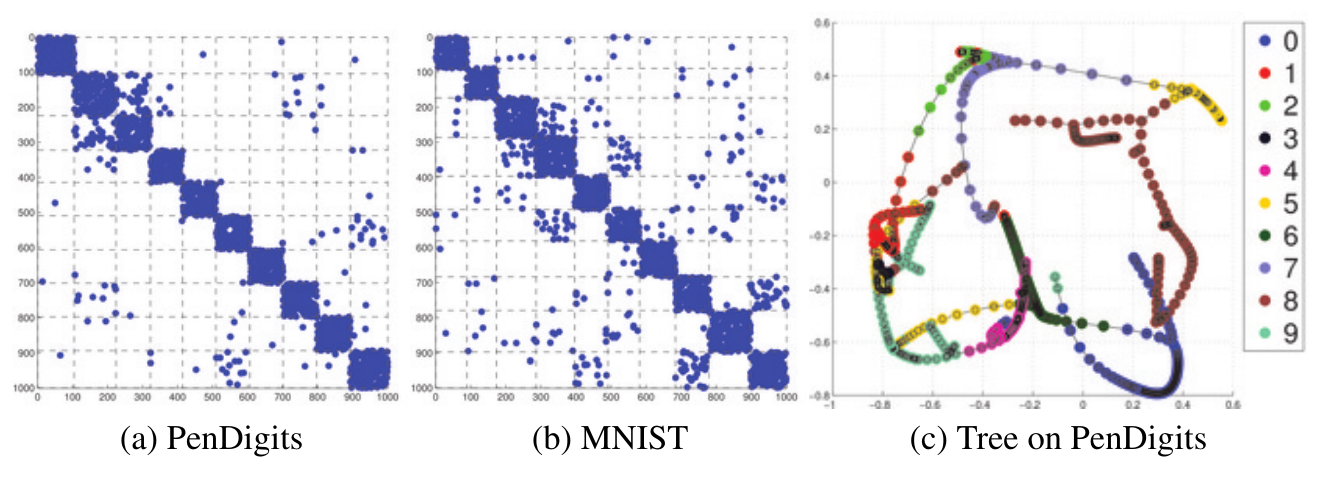} 
	\caption{Experiments on two large-scale datasets. (a)-(b) adjacency matrices of $1000$ centroids after reshuffle according to labels from $0$ t0 $9$ for PenDigits and MNIST, respectively. (c) the learned tree structure of PenDigits over centroids with each label colored (best viewed in colors) is visualized in the first three principal components.} \label{fig:large}
\end{figure}

We use the proposed algorithm for learning an $\ell_1$ graph to find smooth skeletons of handwritten character templates, which can be used to recover the trajectory of the penstroke. The optical characters generally contains loops, self intersections, and bifurcation points. Principal curves have been applied for similar purposes \cite{Kegl2002,Ozertem2011}.  Kegl et al. \cite{Kegl2002} extends their polygonal line algorithm \cite{Kegl2000} to handle loops and self intersections by modifying it with a table of rules and adding preprocessing and postprocessing steps. However, the table of rules has to be specially designed for the self intersections and its difficult to extend to data sets with dimensionality larger than two. Ozertem \cite{Ozertem2011} proposed KDE-SCMS, which can give satisfactory results without any rule or model based special treatment of the self intersections. KDE-SCMS can obtain a bunch of principal points to describe the skeletons of handwritten characters, but there is no explicit structures of these templates. Our proposed method for learning an $\ell_1$ graph can effectively handle loops, self intersections, and disconnected components (possible noises). Moreover, we do not require a set of rules and simultaneously return the smooth skeletons of templates in the form of graphs.

The handwritten digits data set provided by Kegl, which can be downloaded from website\footnote{https://www.lri.fr/$\sim$kegl/researchUdeM/research/pcurves/implementations/SkeletonizationTemplates/}, and is the same dataset reported in \cite{Kegl2002,Ozertem2011}. Following the preprocessing of work \cite{Kegl2002}, we first transform black-and-white character templates into two-dimensional data sets where a unit length of the coordinate system is set to the width and height of a pixel, so that each pixel has integer coordinates. And then, we add the pixel with value $1$ into the data set. Finally, we scale the data set into square area with unit width and height. 

We ran our proposed method for learning a weighted undirected $\ell_1$ graph with one set of parameters for all handwritten characters, that is $\sigma=0.008, \gamma=100, \lambda=0.4$ and $nn=4$.
Fig. \ref{fig:ocr} shows the results of the smooth skeletons on different character templates with various written styles. Most of skeletons are correctly identified as shown in Fig. \ref{fig:ocr}(a). We also observed as shown in Fig. \ref{fig:ocr}(b) that noises in the characters can also be tolerated and distinguished from the skeleton by the proposed method. A small percentage of characters do not achieve a smoothing effect as shown in Fig. \ref{fig:ocr}(c). This problem can be effectively solved by increasing $\sigma$ to $0.02$ as shown in the last row of Fig. \ref{fig:ocr}. All these observations imply that our proposed method for learning an $\ell_1$ graph can effectively deal with loops, bifurcations, self-intersections, and multiple disconnected components. Hence, it is a powerful tool for various problems in the scientific exploratory research.

\subsection{Experiments on Large-Scale Datasets}
We tested our proposed Algorithm \ref{alg:pgl-large} on two large-scale datasets: PenDigits and MNIST\footnote{https://www.csie.ntu.edu.tw/$\sim$cjlin/libsvmtools/datasets/}. 
The PenDigits database is a handwritten digit dataset of 250 samples from 44 writers, which consists a total of $10,992$ images represented by $16$ features using sampled coordination information. The MNIST database of handwritten digits has a total of $70,000$ samples. The digits have been size-normalized and centered in a fixed-size image. Each image of size $28 \times 28$ is represented as $784$ features. We scale each gray pixel to $[0,1]$ by dividing by $255$. Both handwritten digit datasets have labels from $0$ to $9$. In order to visualize the learned graph, we apply PCA to keep $95\%$ of total energy and obtain $D=9$ and $D=154$ for PenDigits and MNIST, respectively. And, we further rescale the projected low-dimensional sample by dividing by the maximum absolute value of each dimension independently. 

We test Algorithm \ref{alg:pgl-large} for learning a tree structure. The same set of parameters, i.e., $\sigma=0.1$ and $ \lambda=0.1$, is fixed for both datasets. Moreover, we apply K-means method to obtain $1,000$ centroids and use these centroids to initialize $\bfZ$ in Algorithm \ref{alg:pgl-large}. The label of each centroid is set by the label using majority voting method from labels of data points which are assigned to this cluster. Fig. \ref{fig:large} shows the results obtained by the proposed method on two datasets. The learned adjacency matrix $\bfW$ is reported in Fig. \ref{fig:large}(a) and \ref{fig:large}(b) for PenDigits and MNIST respectively, where data points are reshuffled according to their labels from $0$ to $9$. We implemented the method in MATLAB and the empirical CPU times spent on learning $\bfW$ are $53.51$ and $222.75$ seconds for PenDigits and MNIST, respectively. We have the following observations: First, the learned adjacency matrices demonstrate a good shape of diagonal matrix. In other words, data points with the same labels are sticked together, which is useful for clustering. In addition, we have explicit tree structure as shown in Fig. \ref{fig:large}(c) drawn in the first three principal components, which further demonstrate the relations among different labels. From Fig. \ref{fig:large}(c), it is clear to see that there is a path started from digits $0$, and passed through $6$, $5$, $9$, and finally to $4$. This is interesting since digits $0$, $6$, $5$, and $9$ are similar due to the bottom half of the images, while digits $9$ and $4$ are similar in terms of the top half of the images. That is, each path on the tree represents certain types of manifold over digits. Furthermore, our proposed Algorithm \ref{alg:pgl-large} can be used for exploratory analysis on a large amount of data points.

\section{Conclusion} \label{sec:conclusion}

In this paper, we proposed a simple principal graph learning framework, which can be used to obtain a set of principal points and a graph structure, simultaneously. The experimental results demonstrated the effectiveness of the proposed method on various real world datasets of different graph structures. Since our principal graph model are formulated for a general graph, the development of principal graph methods for other specific structure is also possible. As a future work, we will explore principal graph learning on other graphs such as $K$-nearest neighbor graphs and apply it to other real-world datasets.

\bibliographystyle{plain}
\bibliography{SGL}

\end{document}